\documentclass[journal]{IEEEtran}

\usepackage{graphicx}
\usepackage{subfigure}
\usepackage{hyperref}
\usepackage{colortbl}
\usepackage[table]{xcolor}
\usepackage{bbm}
\usepackage{xcolor}
\usepackage{amsmath}
\usepackage{amsthm}
\usepackage{amssymb}
\usepackage{makecell}
\usepackage{multirow}
\usepackage{algorithm}
\usepackage{algorithmicx}
\usepackage{algpseudocode}
\usepackage{cases}
\usepackage{booktabs}
\usepackage{threeparttable}
\usepackage{pifont}
\usepackage{cite}
\newtheorem{theorem}{\textbf{Theorem}}

\newtheorem{lemma}{\textbf{Lemma}}
\newtheorem{remark}{\textbf{Remark}}
\newtheorem{assumption}{\textbf{Assumption}}
\newcommand{\etal}{\textit{et al.}}


\begin{document}

\title{ISFL: Federated Learning for Non-i.i.d. Data with Local Importance Sampling \thanks{Zheqi Zhu, Yuchen Shi, and Pingyi Fan are with the Department of Electronic Engineering, Tsinghua University, Beijing 10084, China (e-mail: zhuzq14@tsinghua.org.cn, shiyc21@mails.tsinghua.edu.cn, fpy@tsinghua.edu.cn).

        Chenghui Peng is with the Wireless Technology Laboratory, Huawei Technologies, Shenzhen 518129, China (e-mail: pengchenghui@huawei.com).

        Khaled B. Letaief is with the Department of Electronic and Computer Engineering, Hong Kong University of Science and Technology, Hong Kong (email: eekhaled@ust.hk).

        This work was supported by the National Key Research and Development Program of China (Grant NO.2021YFA1000500(4)). Khaled B. Letaief's work was supported in part by the Hong Kong Research Grants Council under the Areas of Excellence scheme grant AoE/E-601/22-R.}}

% Copyright (c) 2021 IEEE. Personal use of this material is permitted. However, permission to use this material for any other purposes must be obtained from the IEEE by sending a request to pubs-permissions@ieee.org.

\author{Zheqi Zhu, Yuchen Shi, Pingyi Fan,~\IEEEmembership{Senior Member,~IEEE}, Chenghui Peng, and Khaled B. Letaief,~\IEEEmembership{Fellow,~IEEE}}

% \IEEEauthorblockA{\IEEEauthorrefmark{1}
% Beijing National Research Center for Information Science and Technology and Department of Electronic Engineering Tsinghua University, Beijing 100084, China,
% }}

\maketitle
\IEEEpeerreviewmaketitle
% \thispagestyle{empty} % no page number for the first page
% \pagestyle{empty}  % no page number for the second and the later pages

\begin{abstract}
    As a promising learning paradigm integrating computation and communication, federated learning (FL) proceeds the local training and the periodic sharing from distributed clients. Due to the non-i.i.d. data distribution on clients, FL model suffers from the gradient diversity, poor performance, bad convergence, etc. In this work, we aim to tackle this key issue by adopting importance sampling (IS) for local training. We propose importance sampling federated learning (ISFL), an explicit framework with theoretical guarantees. Firstly, we derive the convergence theorem of ISFL to involve the effects of local importance sampling. Then, we formulate the problem of selecting optimal IS weights and obtain the theoretical solutions. We also employ a water-filling method to calculate the IS weights and develop the ISFL algorithms. The experimental results on CIFAR-10 fit the proposed theorems well and verify that ISFL reaps better performance, convergence, sampling efficiency, as well as explainability on non-i.i.d. data. To the best of our knowledge, ISFL is the first non-i.i.d. FL solution from the local sampling aspect which exhibits theoretical compatibility with neural network models. Furthermore, as a local sampling approach, ISFL can be easily migrated into other emerging FL frameworks.
\end{abstract}


\section{Introduction}
\subsection{Background}
% 
% A number of applications have been deployed in systems coupling the communication and computing such as industrial Internet of Things (IoT), vehicular communication, smart city, etc.
As a novel paradigm for distributed intelligent system, federated learning (FL) has received great attention in both research and application areas \cite{mcmahan2017communication}. The core idea of FL is to share the models of distributed clients, which train models using their local data. Basic federated learning systems consist of two stages: 1) Local training, where clients update their local models with local data; and 2) Central aggregation, where the parameter server aggregates the models uploaded from the clients and sends the global models back. Different from the traditional distributed learning methods that split the models or transmit the data, model parameters or gradients are designed to be the elements for client interaction in FL schemes. FL is recognized as a communication-efficient scheme for its full use of the client-center (client-client) communication, clients' computation capability, and the distributed data sources\cite{konevcny2016federated, yang2019federated}.

The structures of FL naturally fit the distributed intelligent systems where computing is coupled with communication \cite{lim2020federated}. On the one hand, FL can serve the hierarchical computing-communication systems such as edge computing or fog computing, where a cloud center and distributed devices/clusters participate. A number of applications emerge especially in the mobile scenarios such as Artificial Intelligence of Things (AIoT) and vehicular internet \cite{niknam2020federated,nguyen2021federated, wu2022mobility}. The potentials of FL have been investigated in distributed inference \cite{li2020federated, zhu2023towards}, system optimization \cite{khan2020federated} and cooperative control \cite{wang2019adaptive, wu2023high}. In particular, due to the efficient sharing mechanism of FL, its combinations with reinforcement learning (RL) or multi-agent reinforcement learning (MARL) have been considered, which is referred to as federated RL (FRL) \cite{wang2019edge}. The corresponding frameworks have been designed for collaboration and scheduling of intelligent edge computing systems in \cite{yu2020deep,zhu2021federated,zhu2023fedlp}. On the other hand, distributed computing systems can also serve federated learning tasks. With the development of wireless communication, intelligent devices and the lightweight machine learning models, FL is envisioned as one of the core technologies in 6G systems \cite{letaief2019roadmap,yang2021federated}. These mobile systems support the FL to be deployed in numerous emerging applications, such as VR/AR, smart city, etc. \cite{yang2019federated1}.
% Overall, FL has shown its strengths on the interaction between computing and communication.

% Recently, FL has been evolving from many aspects. For the basic architecture, apart from the classic horizontal federated learning (HFL), vertical federated learning (VFL) is proposed for the scenarios where the clients possess different features of same data samples \cite{yang2019federated}. Some learning schemes of traditional deep learning areas are also migrated into FL. The typical concepts are the federated transfer learning (FTL) \cite{liu2020secure}, and the aforementioned FRL. Specific FL algorithms have been widely studied, such as the clustered FL \cite{sattler2020clustered}, hierarchical FL \cite{abad2020hierarchical}, decentralized FL \cite{lalitha2018fully}. Furthermore, some theoretical studies on FL exist. A main direction is the convergence analysis. For instance, Stich \cite{stich2018local} derived the convergence of the local-SGD for the first time, and Wang \etal \cite{wang2021cooperative} bounded the gradient norms under several FL settings. Considering the communication factors, the joint analysis and the guide for system design are further studied in \cite{fallah2020personalized,wan2021convergence}. These theoretical analysis guarantees the interpretability, robustness, generalization, privacy protection of FL and makes such learning schemes trustworthy \cite{floridi2019establishing}.
Recently, FL has been evolving from many aspects and some learning schemes of traditional deep learning areas have also been migrated into FL. Furthermore, some theoretical studies on FL arise. For instance, Stich \cite{stich2018local} derived the convergence of the local SGD for the first time, and Wang \etal \cite{wang2021cooperative} bounded the gradient norms under several FL settings. Considering the communication factors, the joint analysis and the guide for system design are further studied in \cite{fallah2020personalized,wan2021convergence}. These theoretical analysis guarantees the interpretability, robustness, generalization, privacy protection of FL and makes such learning schemes trustworthy \cite{floridi2019establishing}.

% fl develop combine h/v alg convergence 

\subsection{Motivation}
% non-iid; solution; sampling, migration; is
Distributed learning systems suffer from the heterogeneity of both the clients and data. Particularly, as mentioned in \cite{kairouz2021advances}, the heterogeneity of distributed data, also referred to as non-i.i.d. (independent and identically distributed) data settings, is one of the most crucial issues in FL. In this work, we consider a commonly acknowledged perspective of FL non-i.i.d. setting, the data category distribution skew, which is typically represented in real world datasets. For example, the label distributions of clients' training data make a difference on the global performance in classification tasks \cite{li2021federated}. Such non-i.i.d. data impacts FL from two aspects: Firstly, the data distributions highly differ for different clients, resulting in the divergence of local models; The second one is the imbalanced local data, which leads to bad generalization and the overfitting of the local models. Therefore, FL on non-i.i.d. data faces major challenges including slow convergence and unsatisfactory performance \cite{hsieh2020non}.

In conventional centralized deep learning, importance sampling (IS) has been regarded as an effective method to mitigate the data imbalance \cite{bugallo2017adaptive}. Both theoretical and experimental analysis has verified its improvement on accuracy \cite{katharopoulos2018not}, as well as training speed \cite{johnson2018training}.
% Meanwhile, a number of centralized deep learning theories and developments from the IS views have been migrated into the context of FL studies. 
In addition to investigating the importance sampling theoretically, these studies also provided new comprehensions of the optimal IS weights. For instance, iterative parameter estimation methods such as Monte Carlo based approaches can be adopted to estimate the weights \cite{doucet2005monte}. By figuring out the factors that determines the optimal sampling, the sampling parameters can be further combined with the RL to automatically learn the adaptive sampling strategies \cite{llorente2021survey}. In the context of FL, it is naturally to consider IS as a potential solution to improve the performance under non-i.i.d. data, which is caused by the imbalanced data distribution. Besides, the IS-based methods can be designed and implemented locally in a distributed manner. Therefore, introducing IS into non-i.i.d. FL is a promising direction without changing the backbone of FL procedures.

% In this paper, in order to address these challenges, we consider a way to improve the non-i.i.d. FL performance with data-level sampling during local updating.
% By doing so, we would be able to require less excessive communication and computation resources compared to the original FL. Specifically, we focus on the cases where each distributed client resamples the training samples from its local data under a set of importance weights. We also derive the theoretical guarantees of optimal IS weights for distributed clients and design the federated updating policies.

\subsection{Related Work}
% usual solution to noniid; isfl
In the literature, some related work on non-i.i.d. FL has been studied from different aspects. The popular solutions include data sharing, augmentation, model distillation, client selection, clustering, and sampling based methods. Here we mainly list the studies tightly related to ours.

As for FL with sampling approach to tackle non-i.i.d. data, Tudor \etal \cite{tuor2020data} proposed a method to identify the data relevance at each client, based on which the data are sampled for local training before the learning tasks start. To extend the dynamic sampling strategies, Li \etal \cite{li2021sample} measured each sample's importance by its gradient norms during the training and designed the FL algorithms with client+sample selection. Such data-driven method requires calculating the instant gradient norm of each training sample, which means that the extra backward propagation shall be processed locally. Another systematic work on importance sampling FL was done in \cite{rizk2020federated}, where both data and client sampling are considered. The importance weights of the mini-batches were derived theoretically and the convergence of algorithms was also analyzed. However, the theoretical guarantees were carried based on the convex-function assumption which is not compatible with deep learning models such as neural networks. Besides, the experiments were tested on a simple regression problem \cite{rizk2021optimal}. Thus, the theoretical analysis of importance sampling based FL in deep learning tasks is still an open problem.

There are other important theoretical results on non-i.i.d. FL. Li \etal \cite{li2019convergence} established the connections between convergence rate and the factors including federated period, gradient Lipschitz, gradient norms, etc. Zhao \etal \cite{zhao2018federated} derived the upper bounds of the parameter deviations between FL on non-i.i.d. data and the centralized training on full data. They also pointed out that the gradient Lipschitz of different categories shall be introduced to describe the non-i.i.d. impacts on the performance gaps. These previous works inspired us to start this work.

\subsection{Contributions \& Paper Organization}
The main contributions of this paper can be summarized as follows:
\begin{itemize}
        \item[$\bullet$] We put forward a generalized {\bf I}mportance {\bf S}ampling {\bf F}ederated {\bf L}earning framework, abbreviated as ISFL, which introduces the local importance sampling to mitigate the impacts of the non-i.i.d. data distribution.
        \item[$\bullet$] We derive the convergence theorem for ISFL, which generalizes the theoretical analysis of original FL studies. 1) ISFL convergence theorem provides a fine-granularity bound that introduces the effects of local importance sampling probabilities. 2) Compared to existing IS based solutions for non-i.i.d. FL, ISFL relaxes the convexity assumptions for better compatibility with deep learning models.
        \item[$\bullet$] We formulate the weight selection of each client as independent sub-problems. The theoretical solutions for the optimal sampling strategies, as well as an adaptive water-filling approach for calculating the optimum IS weights are also presented.
        \item[$\bullet$] We develop the corresponding ISFL algorithms\footnote[1]{The codes in this work are available at \url{https://github.com/Zhuzzq/ISFL}} and carry several experiments based on the CIFAR-10 dataset. It will be shown that the outcomes fit the theorems well and verify the improvement of the proposed algorithms. The experiments also suggest that ISFL is able to approach the accuracy of centralized training. Some insights on the generalization ability, personalized preference, interpretability, as well as data efficiency are also discussed.
\end{itemize}

The rest of the article is organized as follows. In Section \ref{section Framework}, we introduce some necessary preliminaries and the basic idea of ISFL. In Section \ref{section theory}, we derive the theoretical results on the framework analysis and the optimal weighting calculation. The corresponding ISFL algorithms are developed in Section \ref{section alg}. Experiment evaluations and further discussions are presented in Section \ref{section sim}. Finally, in Section \ref{section conclusion}, we conclude this work and point out several potential research directions.

\section{Preliminaries and ISFL Framework}
\label{section Framework}
In this section, we present important preliminaries and introduce local importance sampling into FL. We will then sketch the basic idea and the framework of ISFL. Some key notations are listed in Table \ref{notation}.

\renewcommand\arraystretch{1.3}
\begin{table}[htbp]
        \newcommand{\tabincell}[2]{\begin{tabular}{@{}#1@{}}#2\end{tabular}}
        \centering
        \caption{\upshape Main Notations.}
        \begin{tabular}{c|l}
                \hline
                \rowcolor{gray!20} Notations                         & Description                                             \\
                \hline
                $K$                                                  & The number of the distributed clients.                  \\
                $C$                                                  & The category number of training data.                   \\
                $E_l$                                                & Federated updating period.                              \\
                $\boldsymbol{\pi}=\left\{\pi_1,\cdots,\pi_K\right\}$ & Client weights for federated updating.                  \\
                $\eta$                                               & Learning rate of each client.                           \\
                $\mathcal{D}_k$                                      & Local training set on client $k$.                       \\
                $\xi_i=(x_i, y_i)$                                   & A training data sample of category $i$.                 \\
                $\boldsymbol{p}=\{p_i\}$                             & Global data proportion of category $i$.                 \\
                $\boldsymbol{p}^k=\{p^k_i\}$                         & Local proportion of category $i$ on client $k$.         \\
                $\boldsymbol{w}^k=\{w^k_i\}$                         & IS weights for category $i$ on client $k$.              \\
                $\bar{\boldsymbol{\theta}}_t$                        & Global model at $t$-th epoch.                           \\
                $\boldsymbol{\theta}^k_t$                            & Local model of client $k$.                              \\
                $\ell\left(\xi;\boldsymbol{\theta}\right)$           & Loss function of $\xi$ for model $\boldsymbol{\theta}$. \\

                \hline
        \end{tabular}
        \label{notation}
\end{table}
\subsection{Federated Learning for Non-i.i.d. Data}
By sharing the trained model parameters, FL serves as a communication-efficient paradigm for distributed learning.
The goal of FL is to obtain a global model which performs optimal global loss, i.e.,
\begin{equation}
        \label{g-loss}
        \ell(\bar{\boldsymbol{\theta}}_t):=\mathop{\mathbb{E}}\limits_{\scriptscriptstyle\xi\sim\boldsymbol{p}}\ell\big(\xi;\bar{\boldsymbol{\theta}}_t\big)=\sum_{i=1}^Cp_i\ell\big(\xi_i;\bar{\boldsymbol{\theta}}_t\big).
\end{equation}
The classical FL proceeds by merging the distributed models periodically to obtain the global model:
\begin{equation}
        \label{eq-fl}
        \bar{\boldsymbol{\theta}}_t=\sum\limits_{k=1}^K\pi_k\boldsymbol{\theta}^k_t\qquad{\rm if}\ t \equiv 0\mod{E_l}
\end{equation}
where $\pi_k$ is the weight\footnote[2]{Such weights are mainly considered in client selection. In the context of this work, we focus the weights of local training data.} for client $k$, and $\boldsymbol{\theta}_t^k$ is the local model of client $k$. For FedAvg \cite{mcmahan2017communication}, $\pi_k=\frac{|\mathcal{D}_k|}{\sum_{l=1}^K |\mathcal{D}_l|}$. All clients download the updated global model $\bar{\boldsymbol{\theta}}_t$ every $E_l$ epochs. Since FL conducts a parameter-level fusion, the local training significantly affects the global performance. Each client tends to update the model towards the local optimum which fits its own training data. Thus, one of the most crucial challenges for FL is the non-i.i.d. data distribution, which usually causes the local gradients and model parameters to diverge, as shown in the left of Fig. \ref{noniidsgd}. Compared to centralized training on whole data, the deviation of the distributed clients hinders the global model from reaching the optimum solution using the whole data, which also leads to poor performance, as well as bad convergence.
\begin{figure}[htbp]
        \centering
        \begin{minipage}[b]{0.48\textwidth}
                \includegraphics[width=1\textwidth]{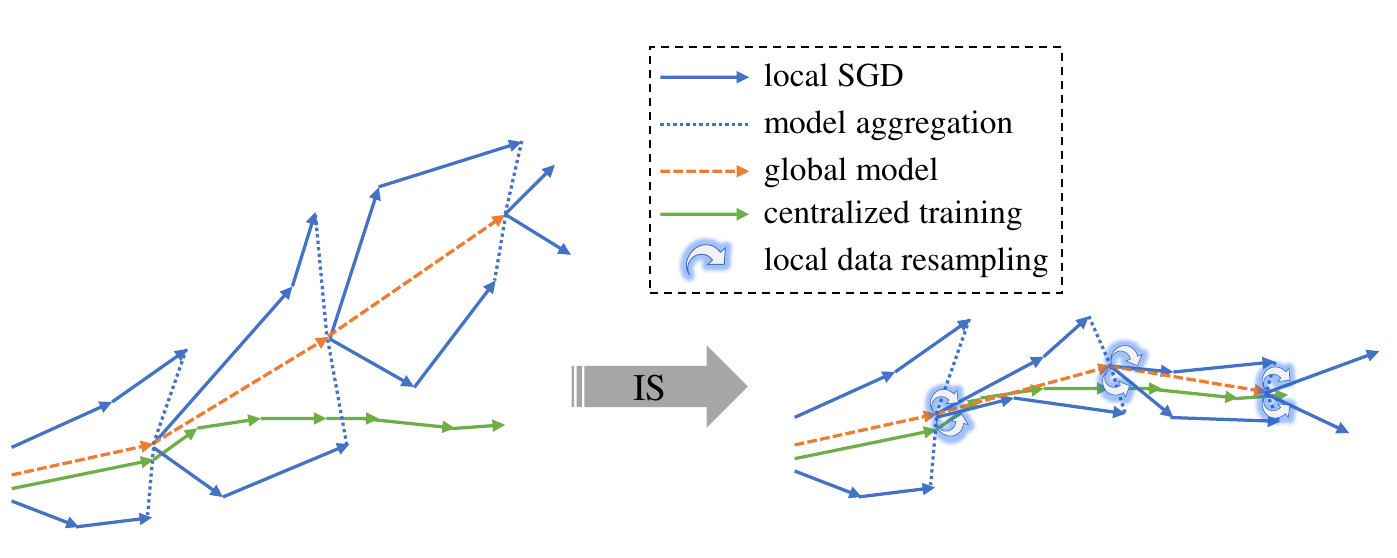}
        \end{minipage}%
        \caption{A sketch for the impact of non-i.i.d. data (w/ and w/o IS): FedAvg with $K=2$ and $E_l=2$. The arrows represent the model evolution.}
        \label{noniidsgd}
\end{figure}

\subsection{Importance Sampling}
Importance Sampling (IS) provides an intuitive solution to make up the gap between the observed distribution (can be obtained from observation) and the target distribution (inherent and latent) \cite{tokdar2010importance}. Consider a certain function $f(x)$ of a random variable $x$, and let $q(x)$ be the latent target distribution and $p(x)$ be the observed one. Through sampling the data according to IS weights $w(x)$ satisfying $\int w(x)p(x)dx=1$, the re-sampled observation can be expressed as:
\begin{equation}
        \label{eq-IS}
        \mathop{\mathbb{E}}\limits_{\scriptscriptstyle x\sim\boldsymbol{p}}\left[w(x)f(x)\right]:=\int p(x)w(x)f(x)dx.
\end{equation}
IS attempts to find proper weights which make the re-sampled observation close to the target expectation, i.e.,
\begin{equation}
        \label{eq-IS2}
        \mathop{\mathbb{E}}\limits_{\scriptscriptstyle x\sim\boldsymbol{q}}\left[f(x)\right]:=\int q(x)f(x)dx.
\end{equation}
Particularly, if $w(x)$ are set ideally as $\frac{q(x)}{p(x)}$, then the observed expectation with IS is equivalent to the objective. In general, iterative algorithms are adopted to estimate the optimal IS weights, as summarized in \cite{luengo2020survey}. Numerous varieties of IS algorithms have been developed and their applications in machine learning area have also been exploited. This inspires us to combine IS with local training to relieve the dilemma caused by non-i.i.d. data. As shown in the right of Fig. \ref{noniidsgd}, by conducting local IS each local epoch, the proper IS weights can guide the local training to be more consistent with the global model, which is expected to improve the performance of non-i.i.d. FL.

\subsection{ISFL Framework}
\begin{figure}[htbp]
        \centering
        \begin{minipage}[b]{0.48\textwidth}
                \includegraphics[width=1\textwidth]{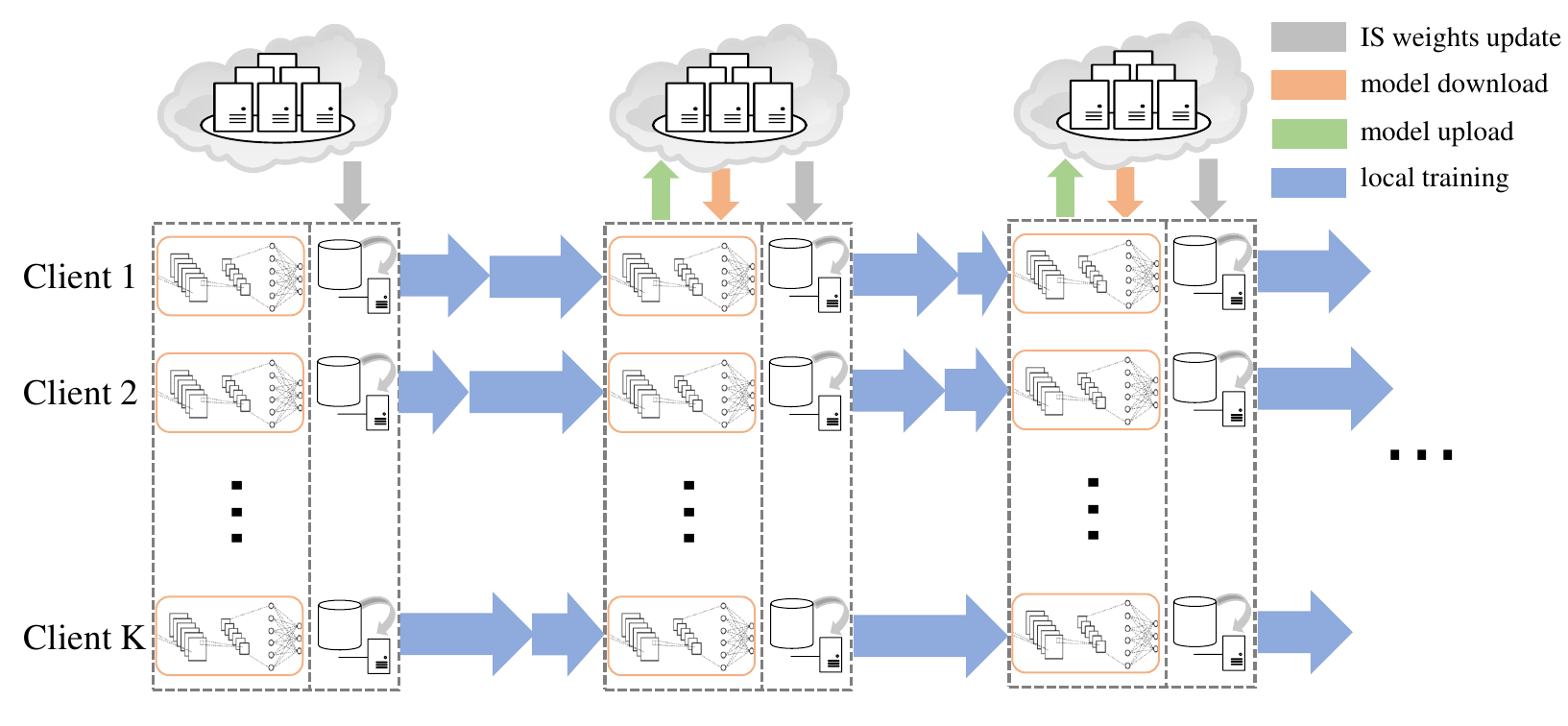}
        \end{minipage}%
        \caption{The workflow of ISFL framework.}
        \label{framework}
\end{figure}
As mentioned, the non-i.i.d. settings and the imbalanced distribution of the local data lead to the model divergence. We shall assume that there exists certain data sampling strategies from which all clients collaborate to make the global model reach a better performance. Hence, we introduce importance sampling for local training on each client and formulate FL with importance sampling, abbreviated as ISFL. In traditional FL, the clients sample their local training data uniformly to carry out the backward-propagation (BP) with the mini-batch SGD based optimizer. Thus, the sampling probabilities are $\boldsymbol{p}^k=\{p^k_i\}$ and the local SGD of a batch in a single iteration on client $k$ can be expressed as:
\begin{equation}
        \label{sgd-p}
        {\boldsymbol{\theta}}^k_{t+1}\leftarrow {\boldsymbol{\theta}}^k_{t}-\eta\sum_{i=1}^Cp^k_i\nabla\ell\big(\xi_i;{\boldsymbol{\theta}}^k_{t}\big).
\end{equation}
In ISFL, clients resample their local training data according to a set of IS weights $\{w^k_i\geq 0\}$, which satisfies $\sum\limits_{i=1}^C p_i^kw^k_i =1 $ for $k=1,\cdots,K$. Then, the local SGD in ISFL shall be changed into:
\begin{equation}
        \label{sgd-q}
        {\boldsymbol{\theta}}^k_{t+1}\leftarrow {\boldsymbol{\theta}}^k_{t}-\eta\sum_{i=1}^Cp^k_iw^k_i\nabla\ell\big(\xi_i;{\boldsymbol{\theta}}^k_{t}\big).
\end{equation}
For simplicity, we use $q^k_i=p_i^kw^k_i$ to represent the re-sampling probabilities in ISFL.
Furthermore, the non-i.i.d dilemma cannot be tackled by simply resampling with the probability $\boldsymbol{p}$ because the model aggregation interacts with the local training. Both procedures are coupled, which together affect the performance of FL. Therefore, the optimal local sampling strategies $\{\boldsymbol{q}^{k*}\}$ for each client are unintuitive and dynamic. The design principle of $\{\boldsymbol{q}^{k}\}$ shall be well formulated.

The basic workflow of ISFL is shown in Fig. \ref{framework}. Clients proceed their local training with the data sampled under the IS weights. The theorems in the next section will demonstrate that the optimal IS weights are related to not only the above two distributions, but also the model parameters. Since the local models change as the training goes on, the IS weights shall be updated iteratively. Thus, ISFL conducts IS weight updating synchronously with the model sharing at every federated round.

\section{Theoretical Results}
\label{section theory}
In this section, we will firstly generalize the FL convergence analysis by introducing local IS to FL. Secondly, based on the convergence theorem, the optimal IS weights are obtained theoretically.

\subsection{Convergence Analysis with Local Importance Sampling}
Similar to most of the convergence theorems, we first state the following assumptions which are commonly exploited to support the theoretical analysis. Particularly, to improve the theoretical compatibility with neural networks, we relax the assumption of the loss functions' convexity or $\mu$-strong convexity, which are widely used in some studies.
\begin{assumption}[Conditional L-smoothness \cite{wang2021cooperative, zhao2018federated}]
    \label{ass1}
    During a given continuous training period $\mathcal{T}$ and for training samples $\xi_i$ of category $i$, the gradient of the model is Lipschitz continuous, i.e., the loss function satisfies:
    \begin{equation}
        \label{ass1eq}
        \left\|\nabla\ell\big(\xi_i;\boldsymbol{\theta}\big)-\nabla\ell\big(\xi_i;\boldsymbol{\theta}'\big)\right\|\leq L_i(\mathcal{T})\left\|\boldsymbol{\theta}-\boldsymbol{\theta}'\right\|,
    \end{equation}
    where $L_i(\mathcal{T})$ is termed as the (empirical) gradient Lipschitz.
\end{assumption}
% :=\max\limits_{\xi_i}\frac{\left\|\nabla\ell\big(\xi_i;\boldsymbol{\theta}\big)-\nabla\ell\big(\xi_i;\boldsymbol{\theta}'\big)\right\|}{\left\|\boldsymbol{\theta}-\boldsymbol{\theta}'\right\|}
Note that there does not exist the invariant gradient Lipschitz constant for deep learning models such as neural networks. Thus, the assumption allows the gradient Lipschitz to vary as the federated updating proceeds.
Furthermore, for given dataset $\mathcal{D}$, the empirical gradient Lipschitz between local model $\boldsymbol{\theta}^k_t$ and global model $\bar{\boldsymbol{\theta}}_t$ at every global epoch can be defined as:
\begin{equation}
    \label{def-eq-L}
    L_{k,i}(t):=\max_{\xi_i\in\mathcal{D}}\ \frac{\Big\|\nabla\ell\big(\xi_i;\boldsymbol{\theta}^k_t\big)-\nabla\ell\big(\xi_i;\bar{\boldsymbol{\theta}}_t\big)\Big\|}{\big\|\boldsymbol{\theta}^k_t-\bar{\boldsymbol{\theta}}_t\big\|},
\end{equation}
when $t=mE_l$.
\begin{assumption}[Bounded SGD variance \cite{stich2018local,wang2021cooperative}]
    \label{ass2}
    The stochastic gradient variance of each model $\boldsymbol{\theta}^k_t$ at epoch $t$ is bounded by:
    \begin{equation}
        \label{ass2eq}
        \mathbb E \left[\big\|\tilde{\boldsymbol{g}}^{k}_t-\boldsymbol{g}^{k}_t\big\|^2\right]\leq \sigma_k^2(t),\qquad \mathbb E\ \|\tilde{\boldsymbol{g}}^{k}_t\|\leq G(t)
    \end{equation}
    where $\tilde{\boldsymbol{g}}^{k}_t$ is the stochastic gradient, $\boldsymbol{g}^k_t=\mathbb E\ \tilde{\boldsymbol{g}}^k_t$ is its expectation. $\sigma_k^2(t)$ and $G(t)$ are the bounds of variance and expectation, respectively.
\end{assumption}
% \begin{assumption}[Gradient diversity measure \cite{haddadpour2019convergence}]
%     \label{ass3}
%     The gradient diversity in FL are defined and bounded as follows:
%     \begin{equation}
%         \label{ass3eq}
%         \Lambda:=\frac{\sum_{k=1}^K\pi_k\|\nabla\ell(\xi;\boldsymbol{\theta}^k_t)\|^2}{\|\sum_{k=1}^K\pi_k\nabla\ell(\xi;\boldsymbol{\theta}^k_t)\|^2}\leq\lambda.
%     \end{equation}
% \end{assumption}
% The gradient Lipschitz and bounded gradient norm are commonly-exploited assumptions in the theoretical studies of FL. \textbf{Through the practical experiments, we find that the so-called gradient Lipschitz constants $\{L_i\}$ actually change with the training process and the model parameters. Thus, the formal notation shall be $\{L_{k,i}(t)\}$. For simple notations, we use $\{L_i\}$ in the theorems (which does not change the theoretical results) and the dynamic gradient Lipschitz will be taken into account in the algorithm designs.}

With the above basic assumptions, we are now ready to investigate the convergence of ISFL in the following theorem.

\begin{theorem}[Convergence analysis]
    \label{thm1-bound}
    Consider the given $T$-step range $\mathcal{T}$ from $T_0$ to $T_1$. By setting proper $\eta$, the expectation gradient norm of ISFL with IS probabilities $\{q^k_i\}$ can be upper bounded by:
    \begin{align}
        \label{thm1-ineq}
        \frac{1}{T}\sum_{t=T_0}^{T_1} & \big\|\nabla\ell(\bar{\boldsymbol{\theta}}_{t})\big\|^2\leq  \frac{2\big(\ell(\bar{\boldsymbol{\theta}}_{T_0})-\ell^*\big)}{\eta T}+\psi(\mathcal{T})\notag \\
                                      & \qquad +\frac{2\eta^2E_l}{T}\sum_{k=1}^K\Big(\pi_k\rho_\mathcal{T}(\boldsymbol{q}^k)\sum_{t=T_0}^{T_1}\phi_k(t)\Big),
    \end{align}
    where
    \begin{equation}
        \label{thm1-eq1}
        \psi(\mathcal{T})=\frac{1}{T}\sum\limits_{t=T_0}^{T_1}\Big[\eta\bar{L}_\mathcal{T}\sum_{k=1}^K\pi_k\sigma_k^2(t) +2CG^2(t)\Big],
    \end{equation}
    \begin{equation}
        \label{thm1-eq2}
        \phi_k(t)=\sum\limits_{\tau=t_c}^{t-1}\Big[2G^2(\tau)+\sigma_k^2(\tau) +\sum\limits_{l=1}^K\pi_l\sigma_l^2(\tau)\Big],
    \end{equation}
    \begin{equation}
        \label{thm1-eq3}
        \rho_{\mathcal{T}}(\boldsymbol{q}^k)=\Big(1+\sum_i(p_i-q^k_i)^2\Big)\Big(\sum_{i} q^k_iL^2_{k,i}(\mathcal{T})\Big).
    \end{equation}
    Therein, $\bar{L}_\mathcal{T}=\sum_i p_i L_{i}(\mathcal{T})$ is the average gradient Lipschitz of the global model, and $t_c=\lfloor\frac{t}{E_l}\rfloor\cdot E_l$ is the latest epoch of model aggregation.
\end{theorem}

The full proof of Theorem \ref{thm1-bound} is presented in Appendix \ref{appendix A}.

\begin{remark}[Interpretation of the bound]
    \label{rmk1}
    Theorem \ref{thm1-bound} extends the FL convergence results by considering the effects of local importance sampling. Similar to some existing studies such as \cite{wang2021cooperative} and \cite{haddadpour2019convergence}, the divergence of gradient norm comes from 3 parts: 1) The first term in RHS of \eqref{thm1-ineq} is the related to the remaining gap with the optimum loss function after previous $T_0$ round federated learning; 2) The second term $\psi(t)$ is the gradient deviation caused by the subsequent model aggregation; 3) The third term mainly comes from the local training with IS and model aggregation during the subsequent progress. Particularly, the modified multiplier $\rho_{\mathcal{T}}(\boldsymbol{q}^k)$ illustrates the influence of the local IS under $\{\boldsymbol{q}^k\}$.
\end{remark}

\begin{remark}[The differences with general convergence bounds of non-i.i.d. FL]
    \label{rmk2-thm1}
    Related studies on the convergence of non-i.i.d. FL have been also investigated. To measure the impacts of the non-i.i.d. data, several statistic metrics were developed, such as 1) the MSE (mean square error) between the local gradients and the global ones, i.e., $\frac{1}{K}\Vert\nabla\ell(\boldsymbol{\theta}^k)-\nabla\ell(\bar{\boldsymbol{\theta}})\Vert^2$ in \cite{wang2021cooperative} and 2) the weighted gradient diversity $\frac{\sum_k \pi_k\Vert\nabla\ell(\boldsymbol{\theta}^k)\Vert^2}{\Vert\sum_k \pi_k\nabla\ell(\boldsymbol{\theta}^k)\Vert^2}$ defined in \cite{haddadpour2019convergence}. Both metrics are bounded as an assumption in their further convergence analysis. The convergence results only imply the relations to the gradient-based statistics, which cannot directly reflect how the data distribution impacts the FL convergence. In contrast, Theorem \ref{thm1-bound} provides a more fine-grained bound involving the proportion of each specific data categories, through which it can be intuitively explored why imbalanced data distribution damages the convergence. Besides, Theorem \ref{thm1-bound} also suggests that the local IS strategies should be designed according to the global data distribution and the gradient Lipschitz, which will be discussed in the next subsection.
\end{remark}

Theorem \ref{thm1-bound} generalizes the original convergence analysis of FL by considering the impacts of local importance sampling. The result is also intuitive since the local importance sampling shows its effect for subsequent local training and model aggregation by multiplying the factor $\rho_{\mathcal{T}}(\boldsymbol{q}^k)$ on the original deviations. Notably, the third term of the upper bound in \eqref{thm1-ineq} is related to clients' local sampling probabilities $\{\boldsymbol{q}^k\}$. Thus, we shall mainly focus on this term in the next to investigate the theoretically optimal sampling strategies.

\subsection{Optimal Importance Sampling Strategies}
The aforementioned theorem implies how local IS probabilities affect the ISFL convergence rate. Based on Theorem \ref{thm1-bound}, we can formulate the optimization problem to minimize the upper bound for better model performance. Firstly, \eqref{thm1-ineq} implies that for given participation set and client weights, the upper bound of the average gradient norm is mainly affected by the local sampling strategies. Then, the selection of the clients and the training samples can be optimized independently and integrated to jointly improve the FL performance on non-i.i.d. data. Secondly, note that with the client weights $\{\pi_k\}$ fixed, only the third term in \eqref{thm1-ineq} is related to the local IS probabilities $\{\boldsymbol{q}^k\}$. Therefore, the solutions to optimal local sampling strategies shall be handled independently on each client. In other words, the general optimization problem of IS can be decomposed into $K$ distributed sub-problems minimizing $\rho_{\mathcal{T}}(\boldsymbol{q}^k)$. Besides, to ensure that samples of each category are sampled at least with a low probability, we set the lowest IS weight $\varpi$ and guarantee the condition $q_i^k\geq \varpi p_i^k$. Then, we are able to formulate the decomposed optimization problems for the $T$-step period $\mathcal{T}$ as follows.
\begin{align}
    \mathcal P^k_{\mathcal{T}}:\quad & \min\limits_{\boldsymbol{q}}\quad \rho_{\mathcal{T}}(\boldsymbol{q}^k)\label{P} \\
                                     & s.t.\quad
    \begin{cases}
        \sum\limits_{i=1}^C q_i^k=1 \tag{\ref{P}{a}}\label{Pa} &                    \\
        q_i^k\geq\varpi\cdot p_i^k                             & \quad i=1,\cdots,C
    \end{cases}
\end{align}

For each client $k$, its optimal local IS probabilities during $\mathcal{T}$ can be obtained by solving the distributed optimization problem $\mathcal P^k_{\mathcal{T}}$. The following theorem gives the theoretical solutions of the optimal IS strategies on each client.
\begin{theorem}[Optimal IS Probabilities]
    \label{thm2-opt}
    For given period $\mathcal{T}$, the optimal probabilities of local importance sampling for client $k$ are:
    \begin{equation}
        \label{thm2-q_star}
        q_j^{k*}=\max\Bigl\{\varpi p_j^k,p_j+\alpha^k_{j}(\mathcal{T})\cdot\Gamma_k(\lambda)\Bigr\},
    \end{equation}
    where
    \begin{equation}
        \label{thm2-alpha}
        \alpha^k_j(\mathcal{T})=\frac{1-\frac{CL_{k,j}^2(\mathcal{T})}{\sum_iL_{k,i}^2(\mathcal{T})}}{\sqrt{\sum\limits_{m=1}^C \Big(1-\frac{CL_{k,m}^2(\mathcal{T})}{\sum_iL_{k,i}^2(\mathcal{T})}\Big)^2}}
    \end{equation}
    is a zero-sum factor representing the gap between the empirical gradient Lipschitz of category $j$ and the average value. $\Gamma_k(\lambda)=\sqrt{\frac{C\lambda}{\sum_iL^2_{k,i}(\mathcal{T})}-1}$ is a factor related to the Lagrangian multiplier $\lambda$.
\end{theorem}

See the proof in Appendix \ref{appendix B-Thm2}.

Theorem \ref{thm2-opt} indicates that the optimal IS probabilities are determined by the global category distribution plus a zero-sum factor related to the gradient Lipschitz. In general, ISFL up-samples the category with a large global proportion or a small gradient Lipschitz. The principle of the optimal local IS strategies can be explained as follows. Two kinds of data categories shall be used more frequently: 1) The categories with larger proportion possess more data samples, which helps the model learn more patterns and improves the local performance; 2) On the other hand, categories with small gradient Lipschitz will benefit the global model convergence at federated epochs. Therefore, in FL for non-i.i.d. data, instead of intuitively resampling the data according to the global distribution (up-sample the smaller proportion and down-sample the higher), we should also consider the offset term in \eqref{thm2prf-qstar}, which reflects the impacts of the gradient deviation and aggregation. Overall, the solutions to the sub-problems $\mathcal P^k_{\mathcal{T}}$ is to resample the local training data with the IS weights $\{q_i^{k*}/p^k_i\}$. Moreover, the exact values of the optimal IS weights can be solved by a water-filling approach as shown in Remark \ref{rmk-thm2}.

\begin{figure}[htbp]
    \centering
    \begin{minipage}[b]{0.46\textwidth}
        \includegraphics[width=1\textwidth]{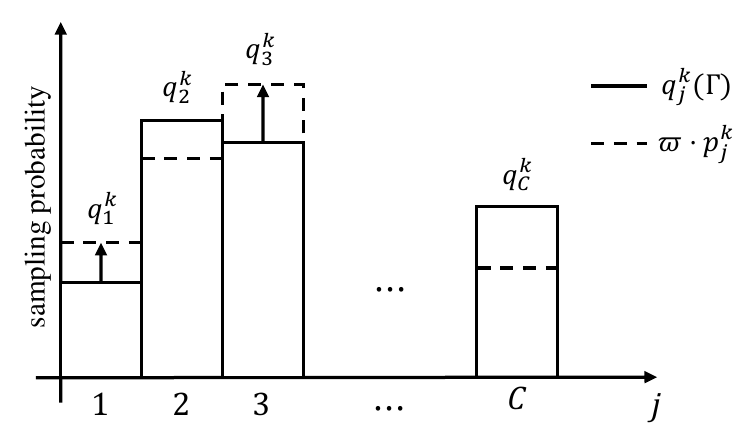}
    \end{minipage}%
    \caption{The water-filling sketch of the optimal IS weights.}
    \label{waterfilling}
\end{figure}

\begin{remark}
    \label{rmk-thm2}
    Let $q_j^{k}(\Gamma)=p_j+\alpha^k_{j}(\mathcal{T})\cdot\Gamma_k(\lambda)$. Note that the $k$-th client's optimal IS probability for category $j$ is the maximum between $\varpi p_j^k$ and $q_j^{k}(\Gamma)$. Besides, $\Gamma_k(\lambda)$ is a monotonic function of $\lambda$ and $\Gamma(\lambda)\geq 0$ always holds. Thus, we only need to consider the optimal non-negative $\Gamma^*$. Then, the numerical solutions in \eqref{thm2-q_star} can be calculated by choosing a proper $\Gamma^*$ through a classical water-filling approach as shown in Fig. \ref{waterfilling}.
\end{remark}
\begin{remark}[The comparison with IS based FL]
    \label{rmk3-thm2}
    A related approach for non-i.i.d. FL is ISFedAvg proposed in \cite{rizk2020federated}, which also adopts the importance sampling to improve the FL performance. However, ISFedAvg only considers the convex loss functions, and assumes that the distance between the local models and the global model is bounded. Besides, its sampling strategy is to up-sample the data with larger gradient norms. In contrast, Theorem \ref{thm2-opt} implies that the categories with larger global proportion or smaller gradient Lipschitz shall be up-sampled. The experimental results in Section \ref{subsec evaluation prop} will indicate that ISFL is more reasonable and compatible for deep learning models. Moreover, the optimal IS weights of ISFL can be calculated by a water-filling approach, which is more flexible and practical.
\end{remark}
By adopting a water-filling approach, we only need to find a proper $\Gamma^*$ to guarantee that the lowest sampling probability $\varpi p_j^k$ falls below $q_j^k(\Gamma)$, which is the water-filling level height of Category $j$. Furthermore, we formulate the following Theorem \ref{thm3-opt-Gamma} to show how the optimal factor $\Gamma^*$ is calculated through the water-filling method.
\begin{theorem}[Calculation of $\Gamma^*$]
    \label{thm3-opt-Gamma}
    For client $k$, if the lowest IS weight $\varpi$ is properly set to satisfy $p_j\geq \varpi p^k_j$ for all categories $j=1,\cdots,C$, through the water-filling approach, the optimal factor $\Gamma^*$ in \eqref{thm2-q_star} shall be chosen as
    \begin{equation}
        \label{thm3-Gamma_star}
        \Gamma^*=\min_j\left\{\frac{p_j-\varpi p_j^k}{-\alpha^k_j(\mathcal{T})}\bigg|\frac{p_j-\varpi p_j^k}{-\alpha^k_j(\mathcal{T})}\geq 0\right\},
    \end{equation}
    if the set is not empty.
\end{theorem}
The proof is given in Appendix \ref{appendix B-Thm3}.

Overall, according to Theorem \ref{thm2-opt} and Theorem \ref{thm3-opt-Gamma}, we theoretically solve $\mathcal P^k_{\mathcal{T}}$ and obtain the optimal local IS strategies for each client.

% lag waterfilling
\section{Algorithm Designs}
\label{section alg}
In this section, based on above theoretical results, we formulate the formal ISFL algorithms. Some practical issues and the costs will also be discussed.
\subsection{ISFL Algorithms}
ISFL generalizes the original FL training process at the federated epoch. In addition to collecting, aggregating and sending back the updated model parameters to clients, the center also calculates and hands out the optimal IS weights for each client. According to Theorem \ref{thm3-opt-Gamma}, we first develop the water-filling algorithm to calculate the optimal IS strategies as Algorithm \ref{alg:water-filling}.
\begin{algorithm}
    \caption{Water-filling Solutions of Optimal Local IS Weights for Client $k$.}
    \label{alg:water-filling}
    {\bf Input:} lowest weight $\varpi$, global distribution $\{p_i\}$, local distribution $\{p^k_i\}$, local empirical gradient Lipschitz $\{L_{k,i}\}$, etc.
    \begin{algorithmic}[1]
        \State $\boldsymbol{\Gamma}_k\leftarrow \emptyset$;
        \For{each category $j=1,\cdots,C$}
        \State Calculate the gradient Lipschitz deviation $\alpha^k_j$ as \eqref{thm2-alpha};
        \If{$\alpha^k_j<0$}
        \State $\boldsymbol{\Gamma}_k\leftarrow\boldsymbol{\Gamma}_k\cup\frac{p_j-\varpi p_j^k}{-\alpha^k_j}$;
        \Else
        \State Continue;
        \EndIf
        \EndFor
        \State Select the optimal $\Gamma$: $\Gamma^*_k=\min\Big\{\boldsymbol{\Gamma}_k\cup\max\big\{\boldsymbol{\Gamma}_k\cup 0\big\}\Big\}$;
        \For{each category $j=1,\cdots,C$}
        \State Calculate the IS weight: $w^k_j\leftarrow\frac{p_j+\alpha^k_j\Gamma^*}{p^k_j}$;
        \EndFor
    \end{algorithmic}
    {\bf Output:} IS weights $\{w^k_i\}$.
\end{algorithm}

In practice, the local empirical gradient Lipschitz changes with the model aggregation since the local models are replaced as new global model every $E_l$ epochs. Therefore, $\{L_{k,i}\}$ and the IS weights shall be updated synchronously with the model aggregation at every federated epoch. When clients finish a set of local training, the gradient Lipschitz will be updated. Meanwhile, the IS weights for next local training epochs will be renewed. The full procedures of FL with local IS, ISFL, will be referred to as Algorithm \ref{alg:isfl}.
\begin{algorithm}[h]
    \caption{ISFL: Federated Learning with Local Importance Sampling.}
    \label{alg:isfl}
    {\bf System Settings:} global round $T_G$, federated period $E_l$, client weights $\left\{\pi_k\right\}$, local dataset $\mathcal{D}_k$, learning rate $\left\{\eta\right\}$, etc.\\
    {\bf Local Initialization:} model parameters $\big\{\boldsymbol{\theta}^k_0\big\}$, IS weights $\{w^k_i\leftarrow 1\}$.
    \begin{algorithmic}[1]
        \State epoch $t\leftarrow 1$;
        \While{$t\leq T_G\cdot E_l$}
        \For{each client $k$}\Comment{\textcolor{cyan}{local training with IS}}
        \State Sample batches $\{\boldsymbol{\xi}^k\}$ according to $\left\{w^k_i\right\}$;
        \State Locally train $\boldsymbol{\theta}^k_t$ with $\{\boldsymbol{\xi}^k\}$;
        % \STATE Sample the training data $\xi^k\sim\left\{w^k_i\right\}$;
        % \STATE $t\leftarrow t+1$;
        \EndFor
        \If{$t \mod E_l==0$}\Comment{\textcolor{orange}{model aggregation}}
        \State Upload local parameters $\big\{\boldsymbol{\theta}^k_t\big\}$;
        \State Update global parameters: $\bar{\boldsymbol{\theta}}_t\leftarrow\sum_{k=1}^K\pi_k\boldsymbol{\theta}^k_t$;
        \For{each client $k$}\Comment{\textcolor{teal}{IS weight updating}}
        \State Update gradient Lipschitz $\left\{L_{k,i}\right\}$ as \eqref{def-eq-L};
        \State Carry Algorithm \ref{alg:water-filling} to update $\left\{w^k_i\right\}$;
        % \State Download IS weights $\left\{w^k_i\right\}$;
        \State Download global parameters: $\boldsymbol{\theta}^k_t\leftarrow\bar{\boldsymbol{\theta}}_t$;
        \EndFor
        \EndIf
        \State $t\leftarrow t+1$;
        \EndWhile
    \end{algorithmic}
    {\bf Output:} Model parameters: $\bar{\boldsymbol{\theta}}_t$ and $\big\{\boldsymbol{\theta}^k_t\big\}$.
\end{algorithm}

Specifically, Line 3 to 6 in Algorithm \ref{alg:isfl} is the local training with importance sampling on each client. Line 8, 9 and 13 comprise the traditional federated parameter updating. The IS weight updating is carried out in Line 11 and 12.
ISFL can be regarded as a modified version of the original FL algorithms with additional operations at the local training and federated updating:
\begin{itemize}
    \item[1)] During local epochs, ISFL introduces the dynamic importance sampling of training data for local training, which re-weights the samples of different categories to combat the impact of non-i.i.d. data.
    \item[2)] At federated rounds, besides the traditional parameter upload/download, ISFL adds extra operations: updating the gradient Lipschitz, calculating the renewed IS weights, and sending them back to clients simultaneously through the parameter downlinks.
\end{itemize}

Through the framework and the algorithms developed for ISFL, we consider the re-sampling as a key operation at the beginning of each local training epoch, which does not change the procedure of the original FL. The IS weights are updated at the end of each federated epoch, which also depends on the inner implementation of the clients or the parameter server. Therefore, the ISFL framework can be easily integrated into the existing FL systems as a key operation block.

\subsection{Practical Issues}
\label{subsec practical}
As discussed above, the local empirical gradient Lipschitz change with the model aggregation and the optimal IS weights shall be updated every federated epoch. In this work, we update the empirical gradient Lipschitz using a small and non-sensitive dataset $\mathcal{D}_L$ according to \eqref{def-eq-L} to protect the privacy of clients in FL.

The deployment of ISFL is flexible. The IS weight updating can be proceeded either centrally at parameter server or locally on each client. Note that the procedures consist of gradient Lipschitz updating and the IS weight calculation, where the former requires excessive computation. Therefore, we prefer to deploy the empirical gradient updating at the parameter server for better system efficiency. With the local data distribution uploaded, the optimal IS weights can be calculated at the server and sent back to each client. Particularly, if the local data distributions are considered to be the clients' privacy, the IS weights shall be updated locally with the necessary factors transmitted instead.

As for the cost of introducing ISFL into original FL, the extra communication costs can be neglected because at each federated rounds, clients only need to download the IS weights or the factors for local calculation, which take up far less capacity compared to the transmission of the model parameters. However, considerable computation costs are inevitable. The main computation consumption occurs while updating the empirical gradient Lipschitz, where the per-sample backward propagation shall be executed for each client using the selected dataset $\mathcal{D}_L$. Thus, the additional cost is $2\times K\times|\mathcal{D}_L|$ times gradient computation at every round of model aggregation. Note that such excessive computation is only decided by the size of $\mathcal{D}_L$ and will not increase as the local data grows.

As a result, the gradient Lipschitz updating at the federated round causes additional computation at the parameter server or distributed clients. Since this work focuses on investigating the ISFL framework, the corresponding computation scheduling and the alternative estimation of the gradient Lipschitz will not be considered in this paper. Both of them are indeed the important topics and also interesting, we may discuss it in the future.

% Above all, we derive the theoretical guarantees based on which the implemention algorithms of the proposed ISFL framework are developed.
\section{Experimental Evaluation}
\label{section sim}
In this section, to evaluate our proposals, we carry out several experiments on image classification tasks. The corresponding FL algorithms are implemented with {\textit{PyTorch}}. The details of the experiment configurations, the results and the insights will be presented below.

\subsection{Experiment Setups}
In the basic experiments, we consider an FL system with multiple clients which fully participate the aggregation. Some basic settings and learning configurations are listed in Table \ref{exptab}. During each global epoch, clients train their models locally for 5 epochs. The local data of each client are divided into batches with size of 128 to carry the mini-batch SGD training. Besides, we update the gradient Lipschitz using a small dataset $\mathcal{D}_L$ by \eqref{def-eq-L} with 500 samples randomly selected from the unused dataset.
\renewcommand\arraystretch{1.3}
\begin{table}[htbp]
        \newcommand{\tabincell}[2]{\begin{tabular}{@{}#1@{}}#2\end{tabular}}
        \centering
        \caption{\upshape Basic experiment settings.}
        \begin{tabular}{c|c}
                \hline
                \rowcolor{gray!20} Parameter                             & Value        \\
                \hline
                Number of clients $K$                                    & 10 \& 20     \\
                Maximum global epoch $T_G$                               & 25           \\
                Federated period $E_l$                                   & 5            \\
                Training batch size $\mathcal{B}$                        & 128          \\
                Learning rate $\eta$                                     & 1e-3         \\
                $N\!R$ (Mixed Non-iid)                                   & 0.95 \& 0.98 \\
                $\alpha$ (Dirichlet Non-iid)                             & 0.2 \& 0.4   \\
                Volume for gradient Lipschitz updating $|\mathcal{D}_L|$ & 500          \\
                \hline
        \end{tabular}
        \label{exptab}
\end{table}
%compared methods
\subsubsection{Dataset and Model}
Our experiments are carried based on two popular image classification datasets, CIFAR-10 and CIFAR-100. Since the goal is to verify the effectiveness of the proposed ISFL schemes, we are not aiming at the state-of-the-art performance. As for the neural network models, we adopt a light-weight ResNet \cite{he2016deep} with 2 residual convolutional blocks and 8 trainable layers in total (ResNet8) for CIFAR-10 tasks. The centralized model trained on full dataset reaches 91.25\% accuracy. While comparing with other non-IS based solutions for non-iid FL under CIFAR-100, we adopt the ResNet9 models for better convergence without complex hyper-parameter tuning or other tricks. The centralized model trained on full dataset gets 62.25\% and 85.25\% accuracy in terms of Top-1 and Top-5, respectively.

\subsubsection{Non-i.i.d. Settings}
We implement two commonly used label-level non-i.i.d. settings, namely, mixed non-i.i.d. as well as Dirichlet non-i.i.d., where the former implies the identical local data size, and the latter covers the diverse local data volume. In detail, mixed non-i.i.d. setting carries out sort-and-partition data split: (a) All the data samples are sorted by their labels and then divided equally into shards. Each shard has 500 images with $(1-N\!R)$ proportion uniformly sampled from all categories; and (b) Each client randomly takes 2 shards. Overall, each client possesses equivalent number of training samples. Larger $N\!R$ leads to more severe data heterogeneity. For Dirichlet non-i.i.d. setting, we generate the local data distribution by the Dirichlet distribution with parameter $\alpha$. Hence, the local data volume is different for each client. The larger $\alpha$ implies more balanced data distribution. The whole data of all clients is referred to as global data, $\mathcal{D}_G$. Since the global data i.e., the ensemble of the clients' data might still be imbalanced, we test the model performance not only on the standard CIFAR-10 test dataset $\mathcal{D}_S$, but also on the global data $\mathcal{D}_G$. The test on $\mathcal{D}_G$ reflects the convergence and training efficiency of FL algorithms. The test on $\mathcal{D}_S$ implies the generalization ability of the model because the data distribution differs from the integrated data for the non-i.i.d. settings. We set $N\!R$=0.95, 0.98 for mixed non-i.i.d. and $\alpha$=0.2, 0.4 for Dirichlet non-i.i.d. settings respectively, to represent different degrees of non-i.i.d. data.

\begin{table*}[htb]
        \caption{\upshape The numerical comparison of different learning schemes on CIFAR-10. The accuracy is reported in terms of the test accuracy on $\mathcal{D}_S$ / $\mathcal{D}_G$.}
        \label{compare_tab}
        \centering
        \setlength{\tabcolsep}{6.1mm}{
                \begin{tabular}{c|cc|cc}
                        \toprule

                        \multirow{2}*{Schemes} & \multicolumn{2}{c|}{{Mixed Non-i.i.d.}} & \multicolumn{2}{c}{{Dirichlet Non-i.i.d.}}                                                                     \\

                        \cmidrule{2-5}         & $N\!R$=0.95                             & $N\!R$=0.98                                & $\alpha$=0.4                    & $\alpha$=0.2
                        \\
                        \midrule
                        % \midrule[0.3pt]
                        FedAvg                 & 63.38 / 66.25                           & 38.39 / 44.47                              & 68.91 / 72.76                   & 59.83 / 62.33                   \\
                        \midrule[0.3pt]
                        Uniform-IS             & 66.10 / 69.64                           & 56.05 / 64.97                              & 67.78 / 71.55                   & 54.52 / 59.94                   \\
                        $p_j$-IS               & 62.97 / 68.23                           & 39.41 / 42.55                              & 69.01 / 72.26                   & 56.74 / 59.28                   \\
                        ISFedAvg               & 48.14 / 49.49                           & 44.67 / 47.17                              & 70.04 / 75.20                   & 69.92 / 73.68                   \\
                        % \midrule[0.3pt]
                        ISFL-$\varpi$=0.05     & \textbf{73.46} / \textbf{79.33}         & 68.58 / \textbf{78.25}                     & \textbf{72.54} / \textbf{78.66} & 68.80 / 74.44                   \\
                        ISFL-$\varpi$=0.01     & 72.18 / 79.23                           & \textbf{69.39} / 77.33                     & 72.08 / 77.92                   & \textbf{68.95} / \textbf{74.50} \\
                        % \multirowcell{2}{FedAvg} & \multirowcell{2}{77.94 / 77.16 / 67.57} & \multirowcell{2}{551.466} & \multirowcell{2}{36.36} \\&&&\\
                        \bottomrule
                \end{tabular}}
\end{table*}

\begin{figure*}[htb]
        \centering
        \subfigure[Local data distribution (D non-iid $\alpha$=0.4).]{
                \label{fig-dist-95}
                \begin{minipage}{0.31\textwidth}
                        \includegraphics[width=1\textwidth]{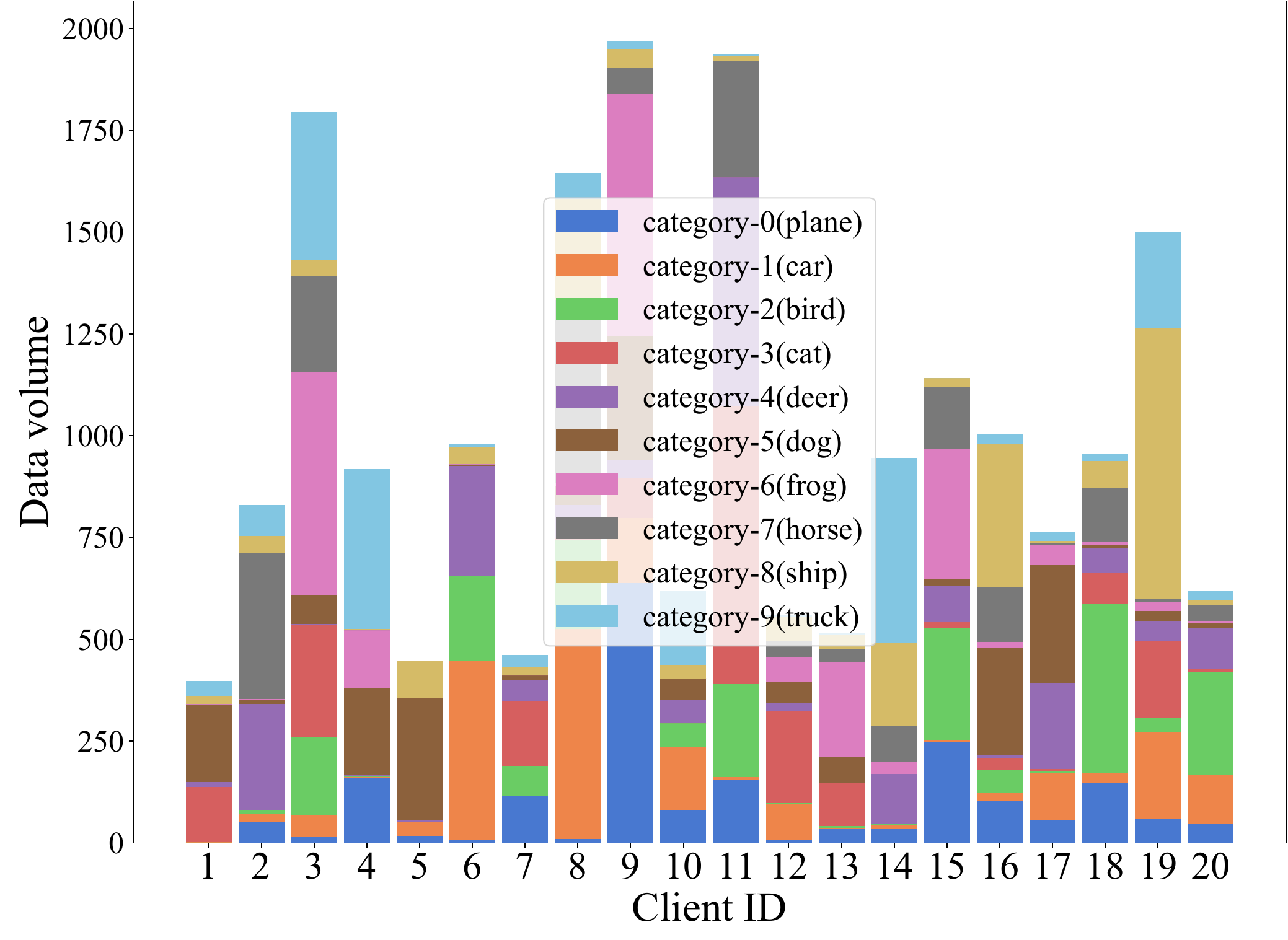}
                \end{minipage}}%
        \subfigure[Test accuracy ($\alpha$=0.4).]{
                \label{fig-acc-95}
                \begin{minipage}{0.31\textwidth}
                        \includegraphics[width=1\textwidth]{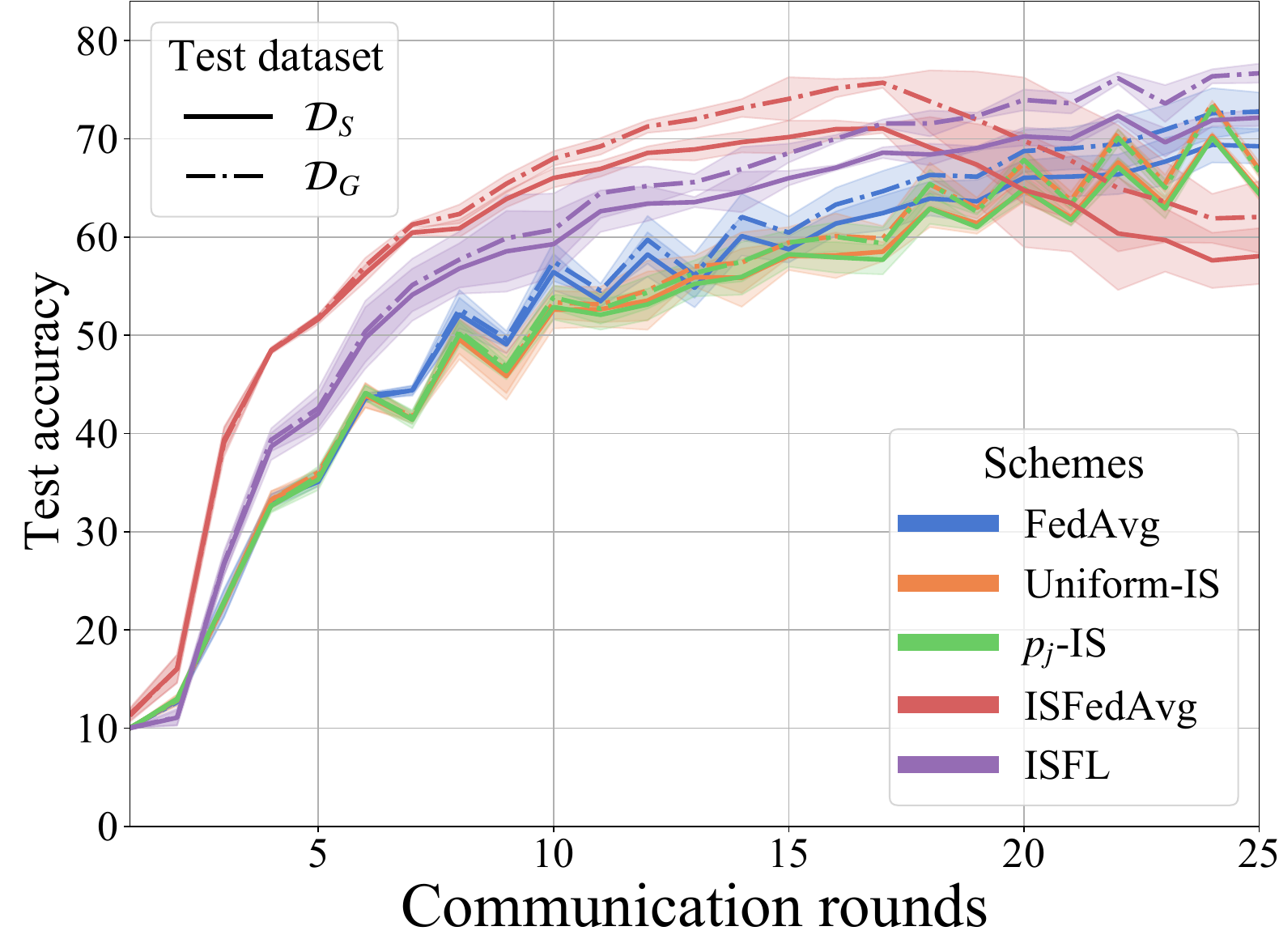}
                \end{minipage}}%
        \subfigure[Convergence of train loss ($\alpha$=0.4).]{
                \label{fig-loss-95}
                \begin{minipage}{0.31\textwidth}
                        \includegraphics[width=1\textwidth]{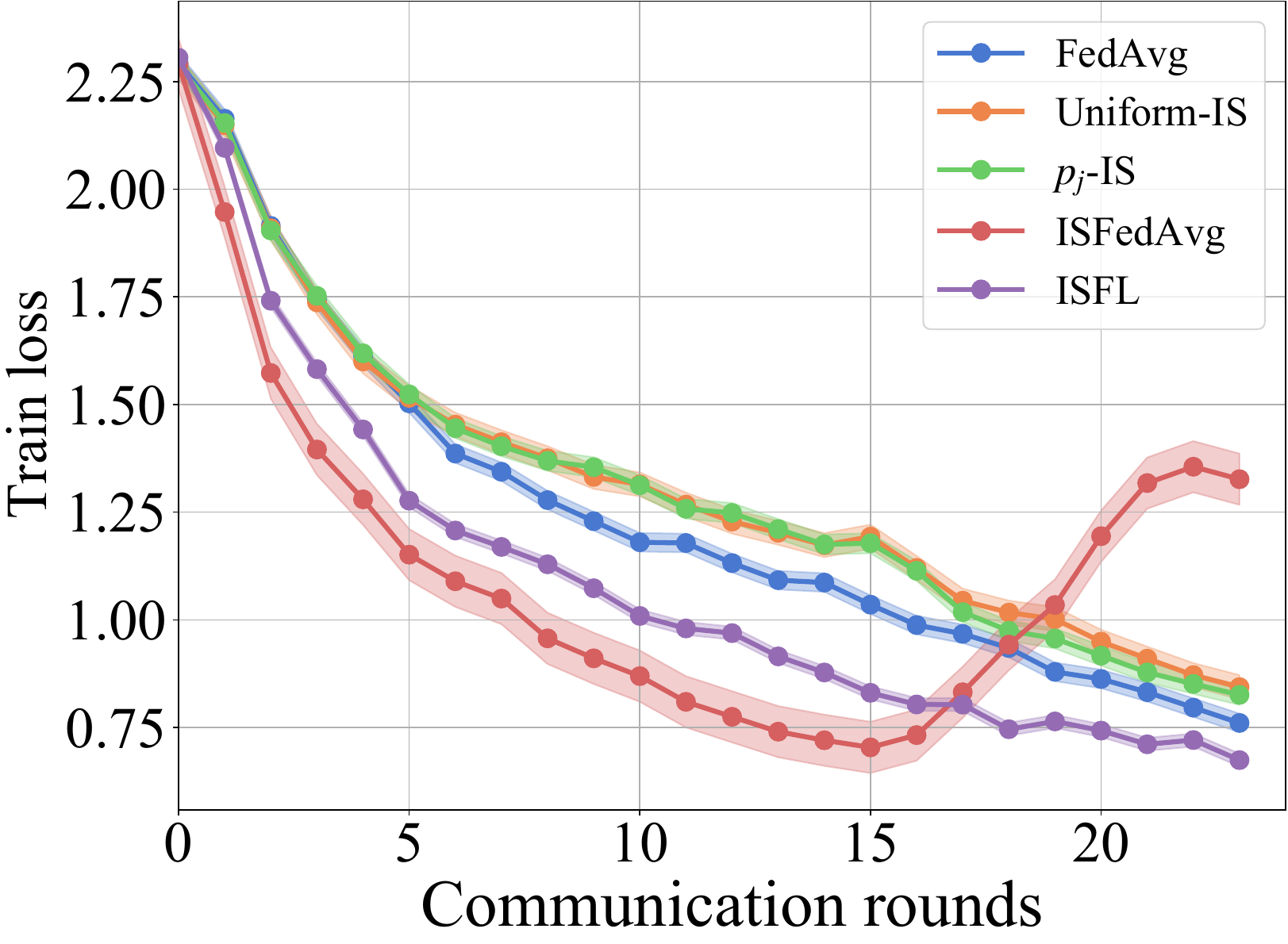}
                \end{minipage}}%

        \centering
        \subfigure[Local data distribution (M Non-iid$N\!R$=0.98).]{
                \label{fig-dist-98}
                \begin{minipage}{0.31\textwidth}
                        \includegraphics[width=1\textwidth]{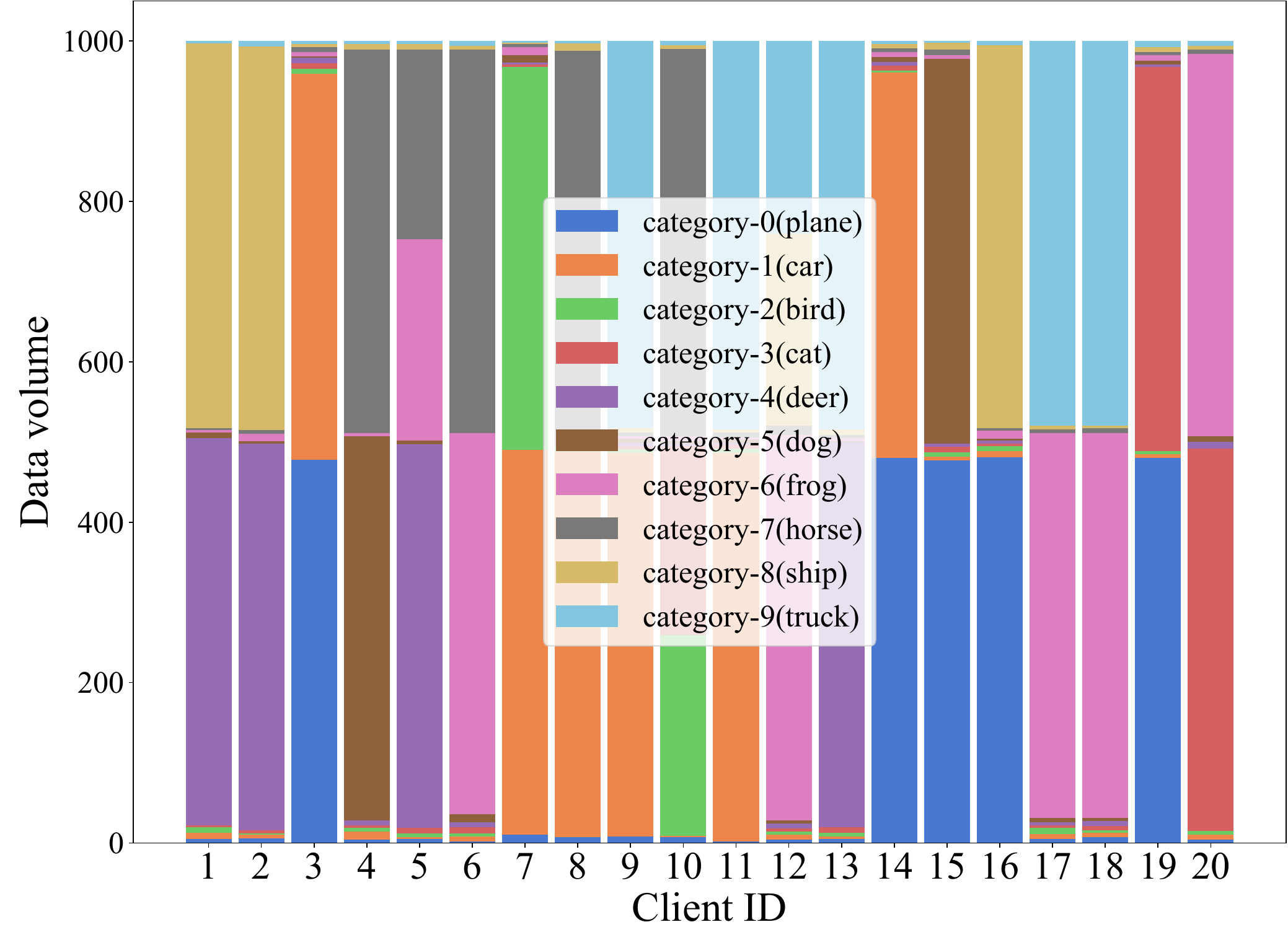}
                \end{minipage}}%
        \subfigure[Test accuracy ($N\!R$=0.98).]{
                \label{fig-acc-98}
                \begin{minipage}{0.31\textwidth}
                        \includegraphics[width=1\textwidth]{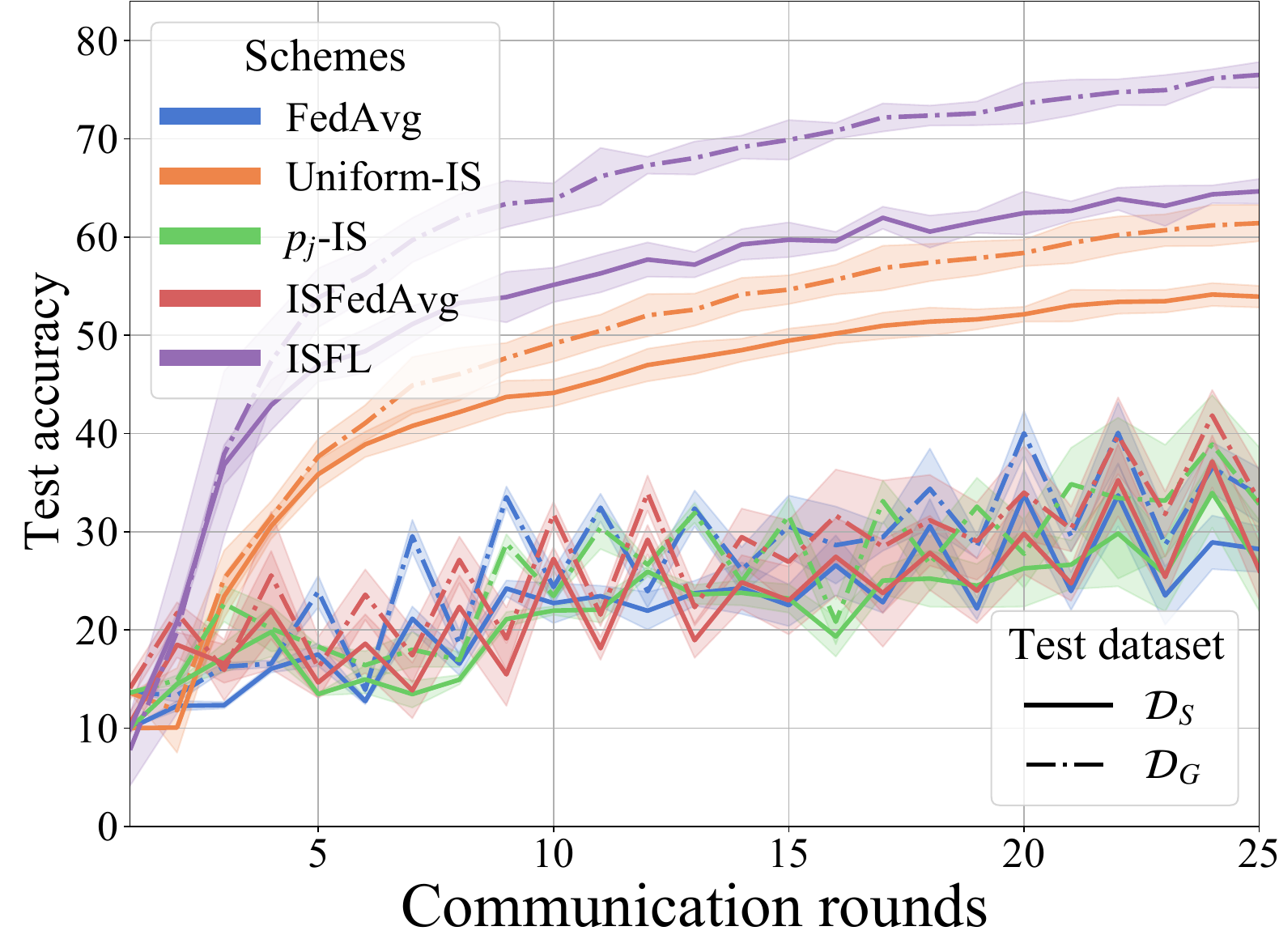}
                \end{minipage}}%
        \subfigure[Convergence of train loss ($N\!R$=0.98).]{
                \label{fig-loss-98}
                \begin{minipage}{0.31\textwidth}
                        \includegraphics[width=1\textwidth]{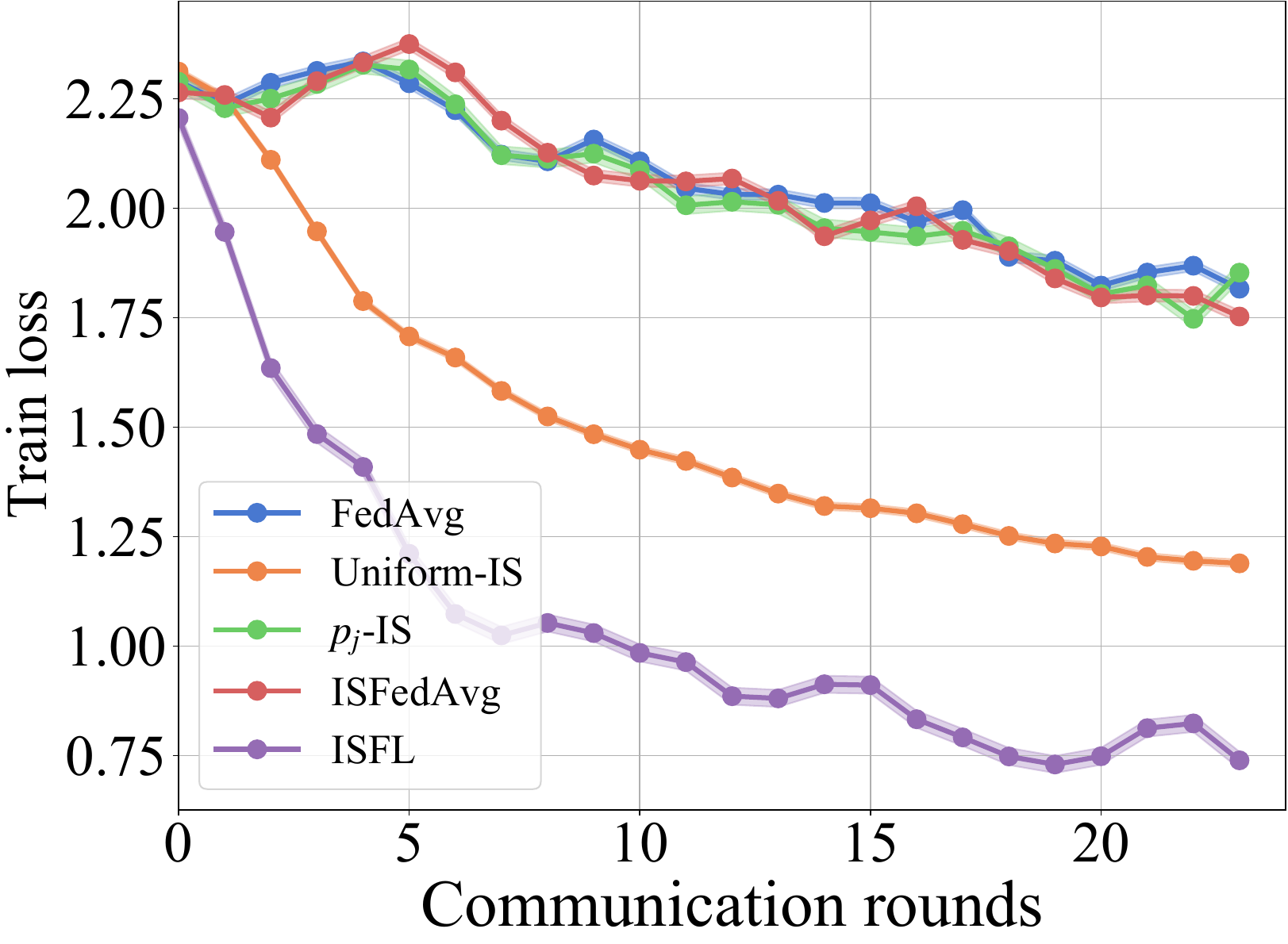}
                \end{minipage}}%
        \caption{Comparisons of several FL schemes on CIFAR-10.}
        \label{fig-compare} %% label for entire figure
\end{figure*}

\subsubsection{Baselines and Comparisons}
In the following experiments, we choose FedAvg as the baseline and implement the importance sampling methods based on it. For comparisons, we select two IS based methods, ISFedAvg and Uniform-IS. ISFedAvg samples local data according to its gradient norms, i.e., more weight is given to a data sample that has a greater gradient norm as proposed in \cite{rizk2021optimal}. Thus, ISFedAvg carries out a per-sample IS weight calculation, which requires more $|\mathcal{D}_G|$ times gradient computation locally on clients for each weight updating. Uniform-IS is an intuitive sampling strategy which re-weights the data and forces to sample the local training data uniformly. As Uniform-IS fixes the sampling weights according to local distribution, no excessive computation is required. Additionally, to evaluate the performance of ISFL solutions, especially the necessity of involving the second biased term $\alpha^k_{j}(\mathcal{T})\cdot\Gamma$ in \eqref{thm2-q_star}, we also carry out an IS-based method which resample the data of each category with its global proportion ${p_j}$, i.e., the first term of the optimal sampling probability in \eqref{thm2-q_star}. We denote such method as $p_j$-IS.

\subsection{Accuracy and Convergence}

% Here we present the experimental results. 
We test two different non-i.i.d. data split settings, i.e., the mixed non-i.i.d. with $N\!R$=0.95/0.98, and the Dirichlet non-.i.d. with $\alpha$=0.4/0.2. We select the Dirichlet non-i.i.d. with $\alpha$=0.4 and the mixed non-i.i.d. with $N\!R$=0.98 to represent the light and severe non-i.i.d. data settings, respectively. The local data distributions are shown in Fig. \ref{fig-dist-95} and \ref{fig-dist-98}.
To evaluate the performance of several training settings and methods, we compare ISFL with the original FL algorithm FedAvg and other sampling based schemes including Uniform-IS, $p_j$-IS, and ISFedAvg. The results are shown in Fig. \ref{fig-compare} and Table \ref{compare_tab}.
The accuracy comparisons of different learning schemes on $\mathcal{D}_S$ and $\mathcal{D}_G$ are presented in Fig. \ref{fig-acc-95} and \ref{fig-acc-98} and the numerical results are listed in Table \ref{compare_tab}. One can observe that for mixed non-i.i.d. data settings, both FedAvg and ISFedAvg perform poorly in terms of accuracy, convergence, and stability. For Dirichlet non-i.i.d. data, ISFedAvg can reach a good accuracy but soon falls into overfitting, as shown in Fig. \ref{fig-acc-95} and \ref{fig-loss-95}. It is worthy to be noticed that though gradient-norm based importance sampling (ISFedAvg) was proved to improve the non-i.i.d. FL for regression models \cite{rizk2021optimal}, such scheme seems not well compatible with the neural network models which do not satisfy the convexity assumption of loss functions. Specifically, assigning higher weights to the data sample with larger gradient norm cannot effectively reduce the divergence of local SGD while employing the NN models. Such phenomena will be explained in detail through the sampling weight discussions later. In contrast, Uniform-IS efficiently improves the accuracy under different non-i.i.d. data settings and performs well on stability. Notably, ISFL enhances FL up to 30\% accuracy under non-i.i.d. data splits. Furthermore, as listed in Table \ref{compare_tab}, $p_j$-IS performs poorly and even lower than original FedAvg, which again verifies that the second term $\alpha^k_{j}(\mathcal{T})\cdot\Gamma$ in \eqref{thm2-q_star} is necessary. Therefore, we deem that the ISFL solutions in Theorem \ref{thm1-bound} to \ref{thm3-opt-Gamma} are valid and effective to non-i.i.d. FL.

% Specifically, ISFL enhances nearly 30\% accuracy on the global data $\mathcal{D}_G$ with good stability. 

The curves and the numerical results suggest that even under severely non-i.i.d. data ($N\!R$=0.98 for mixed non-i.i.d., or $\alpha$=0.2 for Dirichlet non-i.i.d.), ISFL mitigates the performance decays on accuracy, stability and convergence, which strengthens FL's robustness against imbalance data distribution. The above results indicate that ISFL guarantees the generalization ability of the global model even though the local models are trained on non-i.i.d. data. This also implies ISFL's potentials for personalized preference. By substituting the $\{p_j\}$ in \eqref{thm2-q_star} as the distribution that local client prefers, the global model tends to perform well on both standard testing data and the personalized local data.

Besides, ISFL converges most rapidly and reaches the lowest train loss as shown in Fig. \ref{fig-loss-95} and \ref{fig-loss-98}. This also supports that the convergence bound in Theorem \ref{thm1-bound} makes sense and the corresponding optimal IS weights effectively improve the convergence of FL with non-i.i.d. data.

As for the setting of the lowest IS weight $\varpi$, we test the ISFL schemes under $\varpi$=0.05 and $\varpi$=0.01. The results in Table \ref{compare_tab} imply that different $\varpi$ lead to similar model performance. Such outcome is comprehensible since we set $\varpi$ as a small constant in the constraints \eqref{Pa} of $\mathcal P^k_{\mathcal{T}}$ to ensure that all categories could contribute to the federated model at least in a low proportion. Moreover, according to the water-filling solution in Theorem \ref{thm3-opt-Gamma}, the optimal IS probability is only related to the minimal of $\Gamma$. Thus, the small value of $\varpi$ influence little on ISFL performance.

\subsection{Comparisons with Non-IS Based Solutions}
\begin{figure*}[htbp]
        \centering
        \subfigure[CIFAR-10.]{
                \label{fig-cross-10}
                \begin{minipage}{0.31\textwidth}
                        \includegraphics[width=1\textwidth]{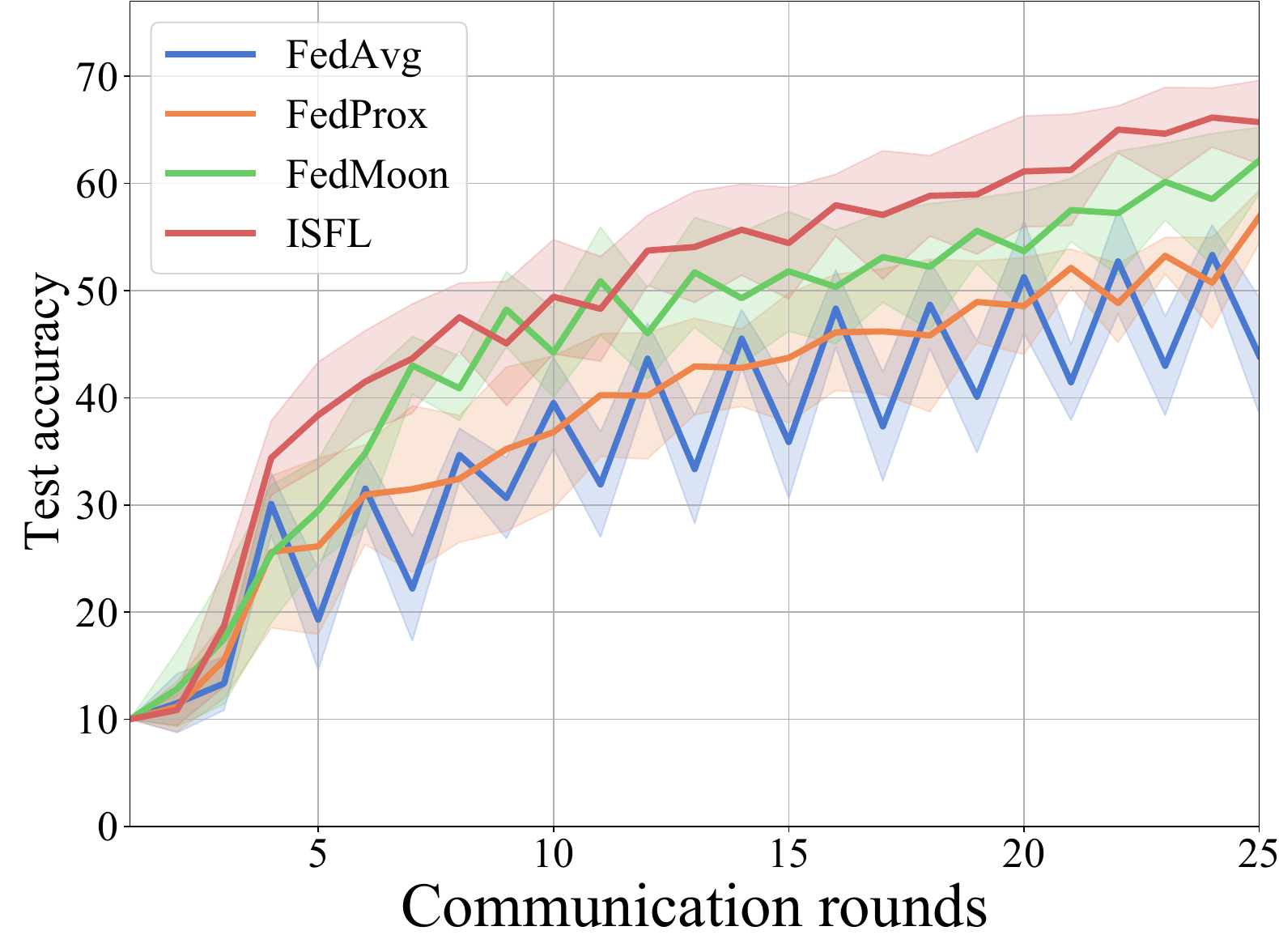}
                \end{minipage}}%
        \subfigure[CIFAR-100 (Top-1 Accuracy).]{
                \label{fig-cross-100-1}
                \begin{minipage}{0.31\textwidth}
                        \includegraphics[width=1\textwidth]{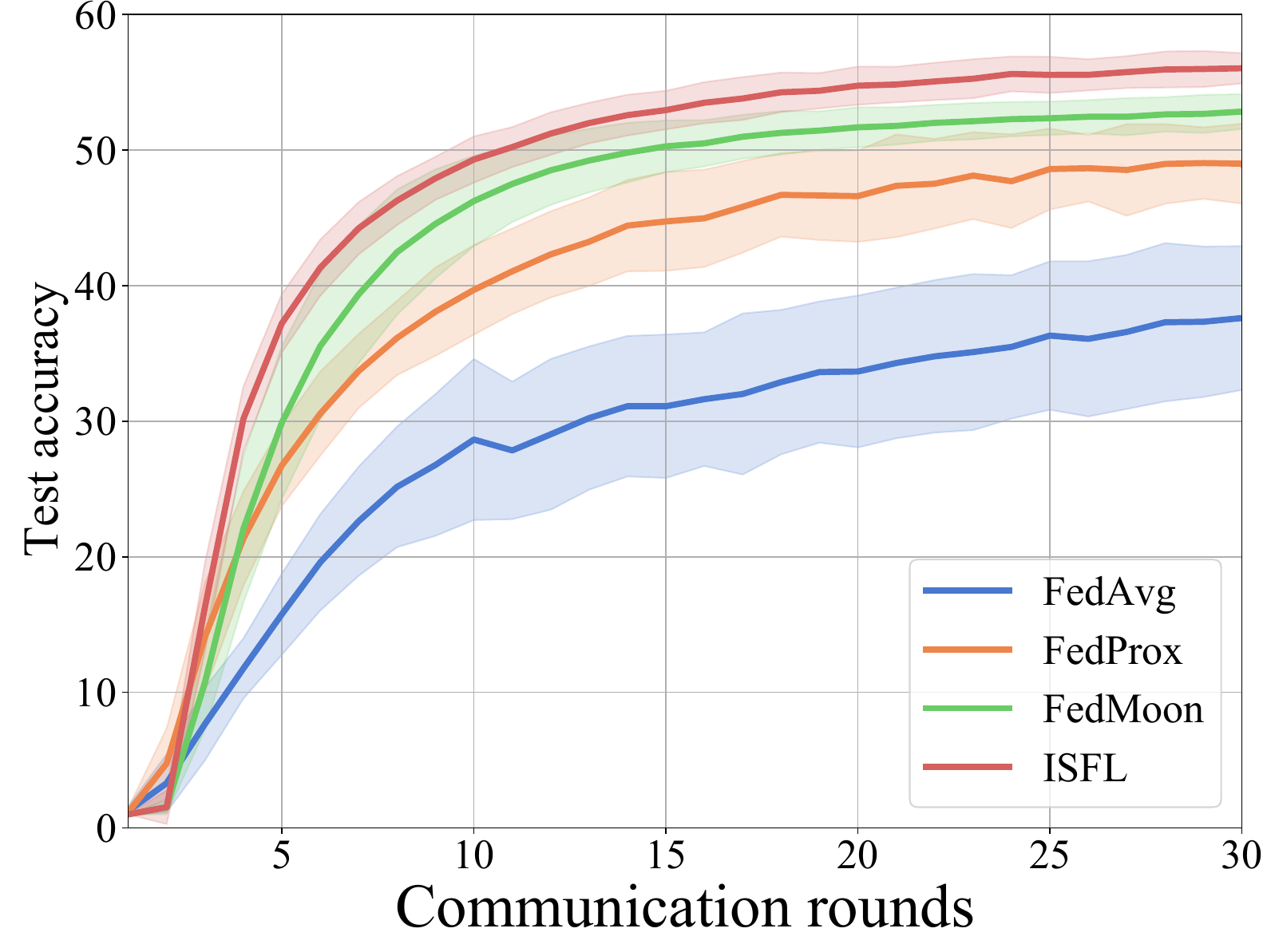}
                \end{minipage}}%
        \subfigure[CIFAR-100 (Top-5 Accuracy).]{
                \label{fig-cross-100-5}
                \begin{minipage}{0.31\textwidth}
                        \includegraphics[width=1\textwidth]{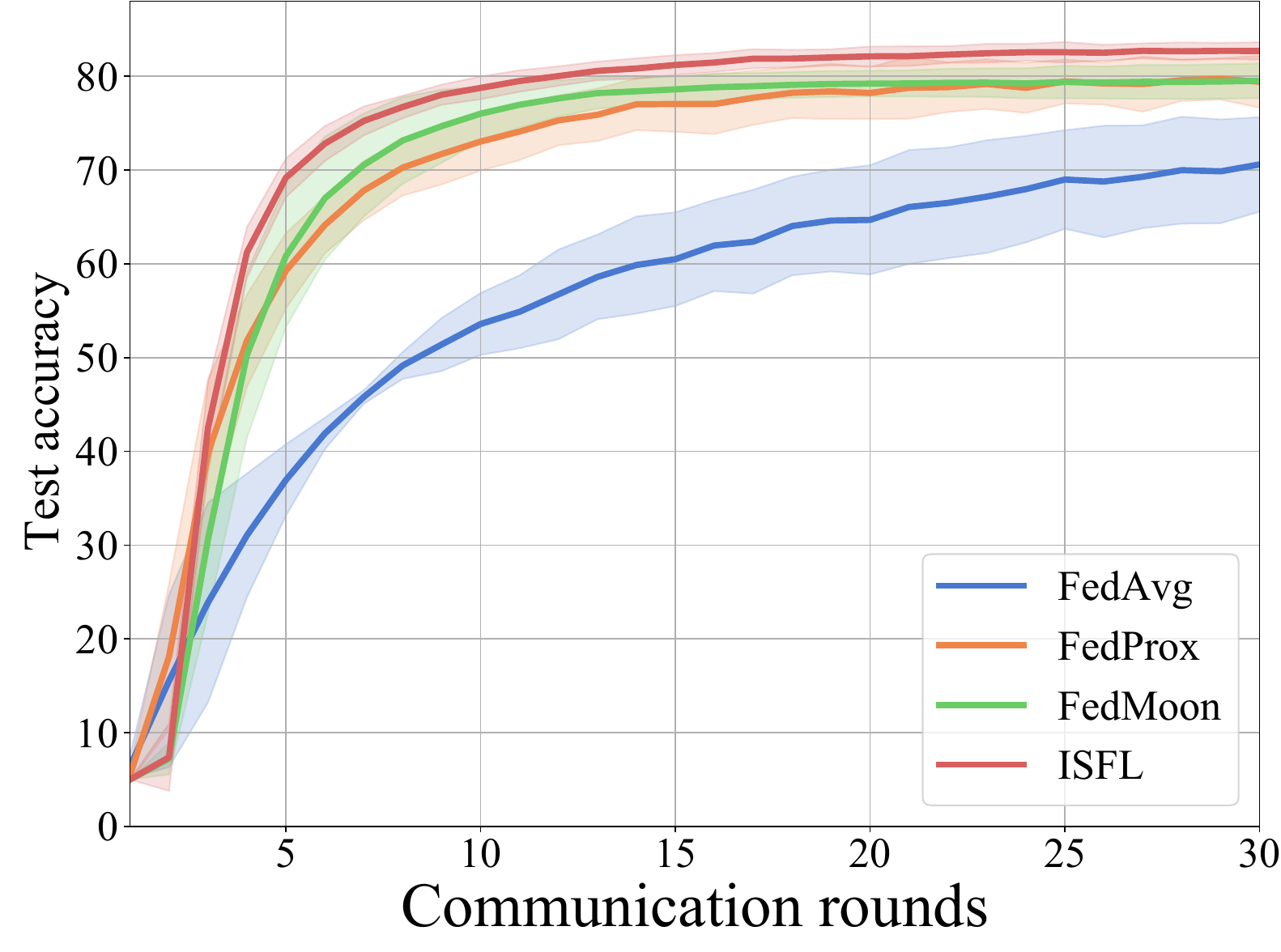}
                \end{minipage}}%
        \caption{Comparisons with non-IS based FL schemes for non-iid settings.}
        \label{fig-compare-cross} %% label for entire figure
\end{figure*}
\begin{table}[htbp]
        \caption{\upshape The numerical comparison with non-IS based solutions to non-i.i.d. FL. $K$=10, Dirichlet Non-i.i.d. $\alpha$=0.2 for CIFAR-10 and 0.5 for CIFAR-100.}
        \label{cross_tab}
        \centering
        \begin{tabular}{c|cc|ccc}
                \toprule

                \multirow{2}*{Schemes} & \multicolumn{2}{c|}{{CIFAR-10}} & \multicolumn{3}{c}{{CIFAR-100}}                                                        \\

                \cmidrule{2-6}         & acc                             & Speedup                         & Top-1 acc      & Top-5 acc      & Speedup            \\
                \midrule
                % \midrule[0.3pt]
                FedAvg                 & 56.09                           & 1$\times$                       & 44.43          & 74.62          & 1$\times$          \\
                \midrule[0.3pt]
                FedProx                & 59.27                           & 1$\times$                       & 52.17          & 79.89          & 3$\times$          \\
                FedMoon                & 65.33                           & 1.79$\times$                    & 54.11          & 80.06          & 4.29$\times$       \\
                ISFL                   & \textbf{69.72}                  & \textbf{2.5$\times$}            & \textbf{57.01} & \textbf{84.64} & \textbf{5$\times$} \\
                \bottomrule
        \end{tabular}
\end{table}

To comprehensively evaluate the effectiveness of ISFL, we compare it with some popular non-IS based FL solutions, FedProx \cite{li2020federatedprox} and FedMoon \cite{li2021model}. These two methods tackle the non-i.i.d. data by adding the regularization terms during the local training. Specifically, FedProx introduces the squared distance between the local model weights and the current global model weights as the regularization term to the loss function, i.e.,
\begin{equation}
        \label{fedprox}
        \ell_{prox}\big(\bar{\boldsymbol{\theta}}^k_{t}\big)=\ell\big(\bar{\boldsymbol{\theta}}^k_{t}\big)+\frac{\mu}{2}\big\|\boldsymbol{\theta}^k_{t}-\bar{\boldsymbol{\theta}}_{t}\big\|^2,
\end{equation}
where $\mu$ is the regularization parameter. FedMoon adopts a contrastive learning mechanism \cite{chen2020simple}, which adds a contrastive loss to the local training loss, i.e.,
\begin{equation}
        \label{fedmoon}
        \ell_{moon}\big(\bar{\boldsymbol{\theta}}^k_{t}\big)=\ell\big(\bar{\boldsymbol{\theta}}^k_{t}\big)+\mu\ell_{con}\big(\boldsymbol{\theta}^k_{t},\boldsymbol{\theta}^k_{t-1},\bar{\boldsymbol{\theta}}_{t};\tau\big),
\end{equation}
where the contrastive loss is defined as:
\begin{equation}
        \label{conloss}
        \ell_{con}=-\log\frac{\exp\big(\text{sim}(z^k_{t},z_{t})/\tau\big)}{\exp\big(\text{sim}(z^k_{t},z_{t})/\tau\big)+\exp\big(\text{sim}(z^k_{t},z^k_{t-1})/\tau\big)}.
\end{equation}
Therein, $\text{sim}(\cdot)$ is the cosine similarity function, and $z$ represents the intermediate output before the output layer.
The comparisons are conducted on CIFAR-10 and CIFAR-100 datasets under Dirichlet non-i.i.d. data settings with $\alpha$=0.2 and 0.5, respectively. The number of the participation clients is set as 10.
After numerous experiments, we select the best hyper-parameters ($\mu$, $\mu$/$\tau$) for each method: FedProx ($\mu$=0.2), FedMoon ($\mu$=1, $\tau$=0.5) for CIFAR-10; and FedProx ($\mu$=0.01), FedMoon ($\mu$=1, $\tau$=0.2) for CIFAR-100.

The results are shown in Table \ref{cross_tab} and Fig. \ref{fig-compare-cross}. The experimental results demonstrate that ISFL outperforms the non-IS based solutions in terms of accuracy and convergence rate. ISFL achieves 69.72\% Top-1 accuracy on CIFAR-10 and 57.01\% Top-1 accuracy on CIFAR-100, which are 4\% higher than the best non-IS based solutions. Moreover, ISFL converges faster and reaches the lowest train loss. We compare the training acceleration of each method to reach the FedAvg performance. The speedup of ISFL is also significant, which is 2.5$\times$ on CIFAR-10 and 5$\times$ on CIFAR-100. The results indicate that ISFL is an effective and efficient FL solution for non-i.i.d. data.

\begin{figure}[htbp]
        \centering
        \begin{minipage}[b]{0.46\textwidth}
                \includegraphics[width=1\textwidth]{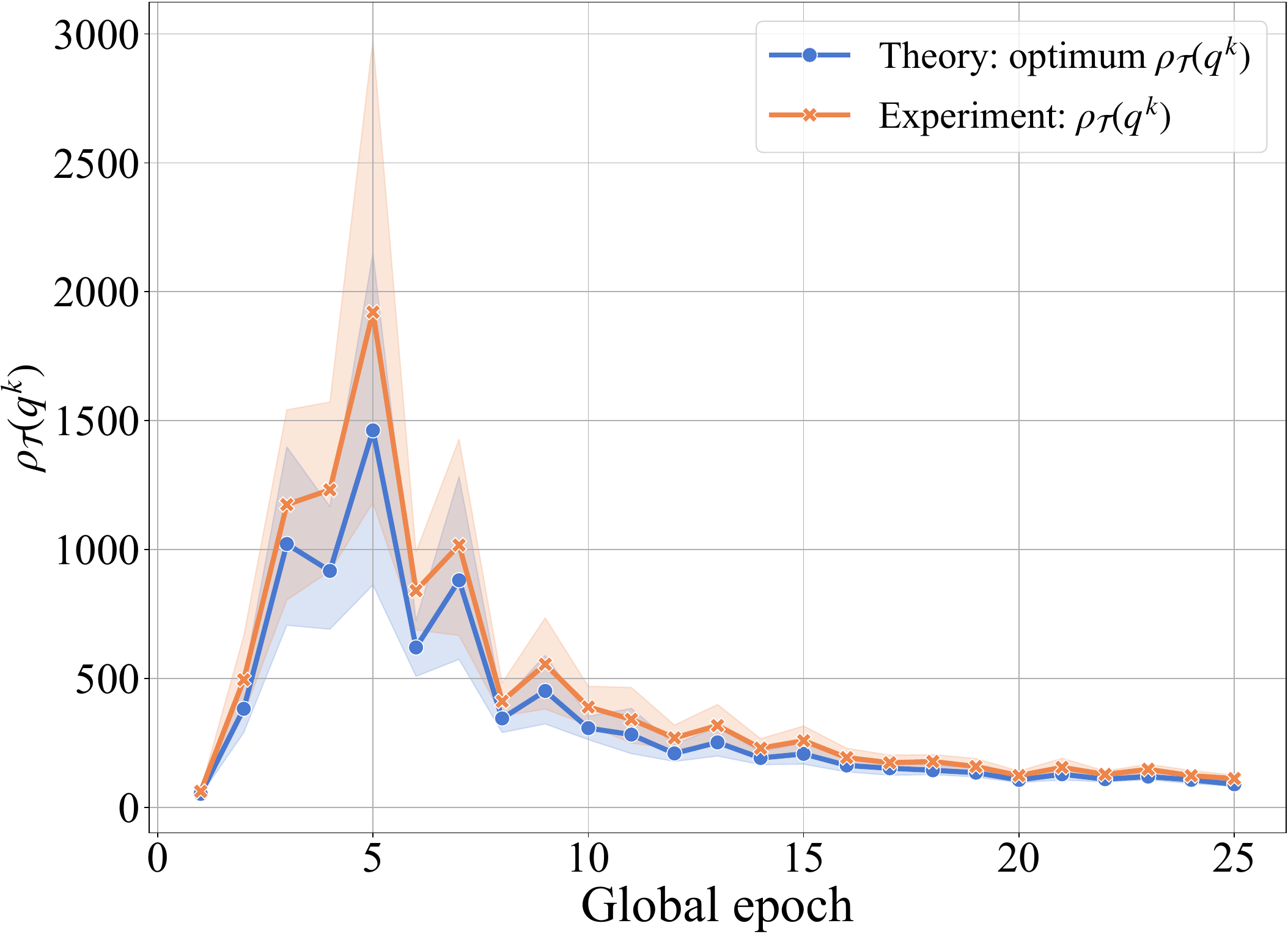}
        \end{minipage}%
        \caption{The evaluation of the convergence bound and the solution ($N\!R$=0.98).}
        \label{fig-thm-val}
\end{figure}

\subsection{Evaluation of The Propositions}
\label{subsec evaluation prop}

\begin{figure*}[htbp]
        \centering
        \subfigure[ISFedAvg.]{
                \label{fig-gnis-q}
                \begin{minipage}{0.47\textwidth}
                        \includegraphics[width=1\textwidth]{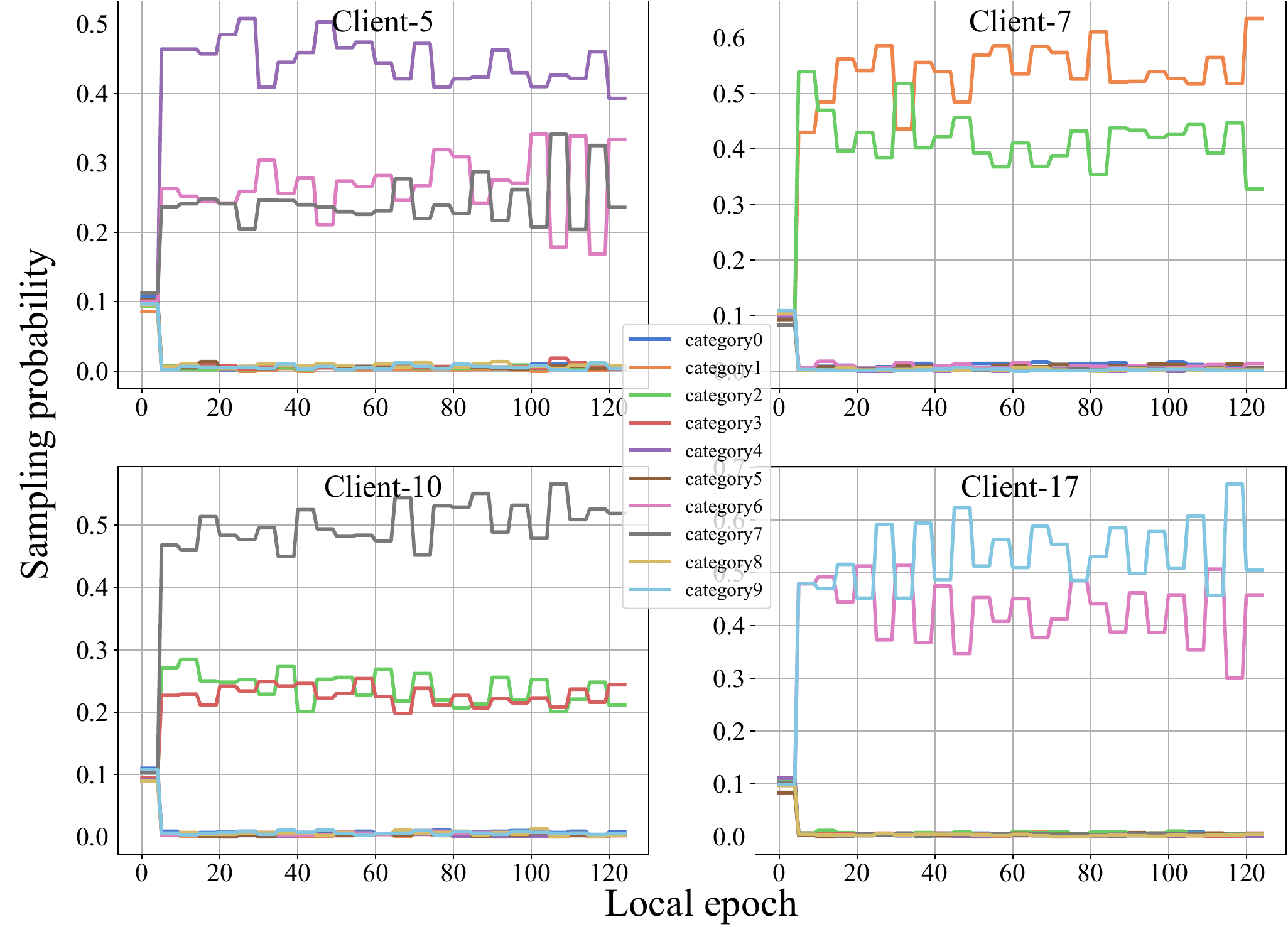}
                \end{minipage}}%
        \subfigure[ISFL.]{
                \label{fig-is-q}
                \begin{minipage}{0.47\textwidth}
                        \includegraphics[width=1\textwidth]{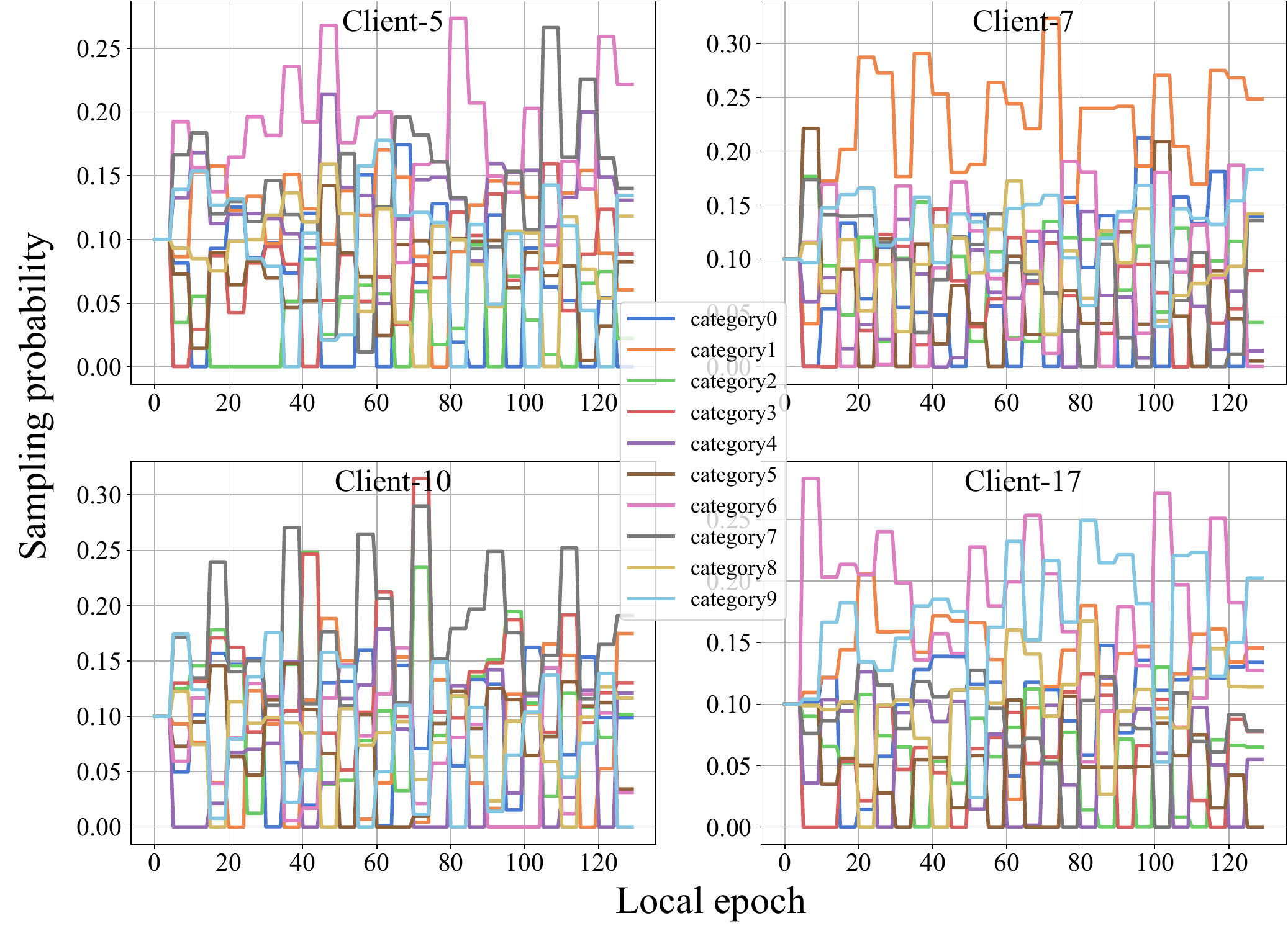}
                \end{minipage}}%
        \caption{Importance sampling probability $\{q^k_i\}$ of the selected clients ($N\!R$=0.98).}
        \label{fig-q}
\end{figure*}

Then we implement further experiments to evaluate the theoretical results and explain the local sampling strategies of ISFL.

To validate the core theorems of ISFL, we investigate the evolution of the multiplier $\rho_{\mathcal{T}}(\boldsymbol{q}^k)$ in Theorem \ref{thm1-bound}, which involves the effects of carrying out local importance sampling. The theoretical minimum of $\rho_{\mathcal{T}}(\boldsymbol{q}^k)$ and the real $\rho_{\mathcal{T}}(\boldsymbol{q}^k)$ solved by Theorem \ref{thm2-opt}-\ref{thm3-opt-Gamma} are plotted in Fig. \ref{fig-thm-val}, respectively. Both theoretical minimum $\rho_{\mathcal{T}}(\boldsymbol{q}^k)$ and experimental $\rho_{\mathcal{T}}(\boldsymbol{q}^k)$ decrease with the global epoch. This verifies the reasonability of Theorem \ref{thm1-bound}: Local importance sampling with optimal weights $\{w^k_j\}$ reduces the multiplier $\rho_{\mathcal{T}}(\boldsymbol{q}^k)$ in \eqref{thm1-ineq}. In this way, the upper bound of the gradient norm is minimized, and the convergence is strengthened.
Moreover, it is easy to confirm that the theoretical optimum and the experimental values of ISFL show same tendency. Two curves get closer and converge stably to a small value as the progress carries on. This indicates that the theoretical optimum $\rho_{\mathcal{T}}(\boldsymbol{q}^{k*})$ and the proposed solutions are consistent. Thus, we deem that the problem $\mathcal P^k_{\mathcal{T}}$ is well solved by Theorem \ref{thm2-opt} and Theorem \ref{thm3-opt-Gamma}.

Next, we conduct client-level analysis of importance sampling weights. With the random seeds fixed, Fig. \ref{fig-q} presents the evolution of the local sampling probability $\{q^k_i\}$ under ISFedAvg and ISFL on selected clients. Fig. \ref{fig-gnis-q} clearly reveals why gradient norm based IS schemes fail for NN models. Though the categories with large gradient norms are upper-sampled for training, the gradients of these categories still remain large norm. In other words, especially for FL with neural network models, utilizing a data sample more frequently in local training does not guarantee that its gradient on the updated model will reduce in future local training, which is also because periodical aggregation wipes away the efforts that the local models take to reduce the gradients of training samples. Therefore, we tend to explain that sampling methods based on GradNorm might not be well compatible with the FL systems using deep learning models such as NNs. On the contrary, as shown in Fig. \ref{fig-is-q}, ISFL seems to carry out a proper sampling strategy and the sampling probability curves seem more i.i.d. as well as adaptive patterns. According to Algorithm \ref{alg:isfl}, ISFL updates the local IS weights synchronously with the communication rounds. Each client adopts its latest IS weights for the local training epochs. One can observe that the IS probabilities of different categories change dynamically. As shown, ISFL does not simply carry the re-weighting (Uniform-IS) to sample the data of each category uniformly. Some categories are always sampled with a large probability, while some are occasionally re-weighted to nearly 0 probability. Moreover, the experiments illustrate that some categories with large proportion will be down-sampled but still get a large IS probability than uniform sampling, e.g., Category 1 for Client 7. Meanwhile, some categories with small proportion will be highly up-sample such as Category 3 for Client 10. The experimental results meet the conclusion in Theorems \ref{thm2-opt} and \ref{thm3-opt-Gamma}: According to \eqref{thm2-q_star}, the categories with larger global proportion or smaller empirical gradient Lipschitz might be utilized more frequently in local training.

\begin{figure}[htbp]
        \centering
        \begin{minipage}[b]{0.48\textwidth}
                \includegraphics[width=1\textwidth]{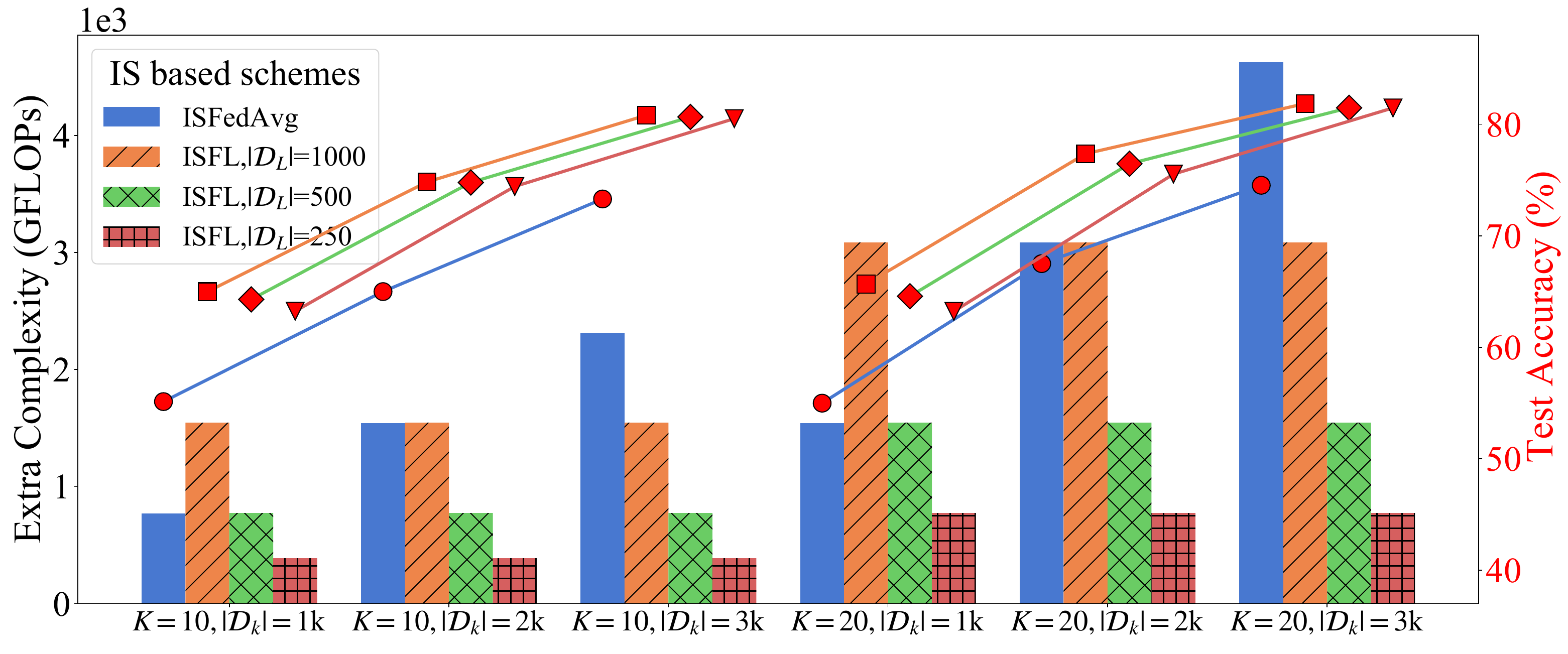}
        \end{minipage}%
        \caption{The extra-computation vs. accuracy analysis with different $K$, $|\mathcal{D}_k|$ and $|\mathcal{D}_L|$.}
        \label{fig-cmplx}
\end{figure}
Finally, as discussed in Section \ref{subsec practical}, we analyze the extra-computation of ISFL. The extra-computation is defined as the additional computation cost to calculate the sampling weights. The extra-computation vs. model-performance results are shown in Fig. \ref{fig-cmplx}, where the experiments are conducted under mixed non-i.i.d. data and different settings of client number $K$, local data volume $|\mathcal{D}_k|$, and Lipschitz data volume $|\mathcal{D}_L|$. Since ISFL adopt a small set $\mathcal{D}_L$ to compute the category-wise weights, the extra-computation does not increase with the local data volume. The comparison on $|\mathcal{D}_L|$ also reflects that a small data set is enough to update the gradient Lipschitz of each category. As shown by the accuracy curves, ISFL with different $|\mathcal{D}_L|$ maintains similar high accuracy even if the $|\mathcal{D}_L|$ is reduced from 1000 to 250.
This is a significant advantage compared to sample-wise IS methods like ISFedAvg, especially for the scenarios with limited computation resources or massive data.
This indicates that the ISFL scheme is efficient, robust as well as practical for FL systems with limited computation resources.

% means that the parameter deviation $\boldsymbol{\Delta}_t$ keeps its scale at the federated epochs. The deviation curves between each local model and the global model in Fig. \ref{fig-pd} confirm that the model deviation and the convergence of ISFL are well tackled.

\subsection{Sampling Efficiency}
\begin{figure}[htbp]
        \centering
        \begin{minipage}[b]{0.48\textwidth}
                \includegraphics[width=1\textwidth]{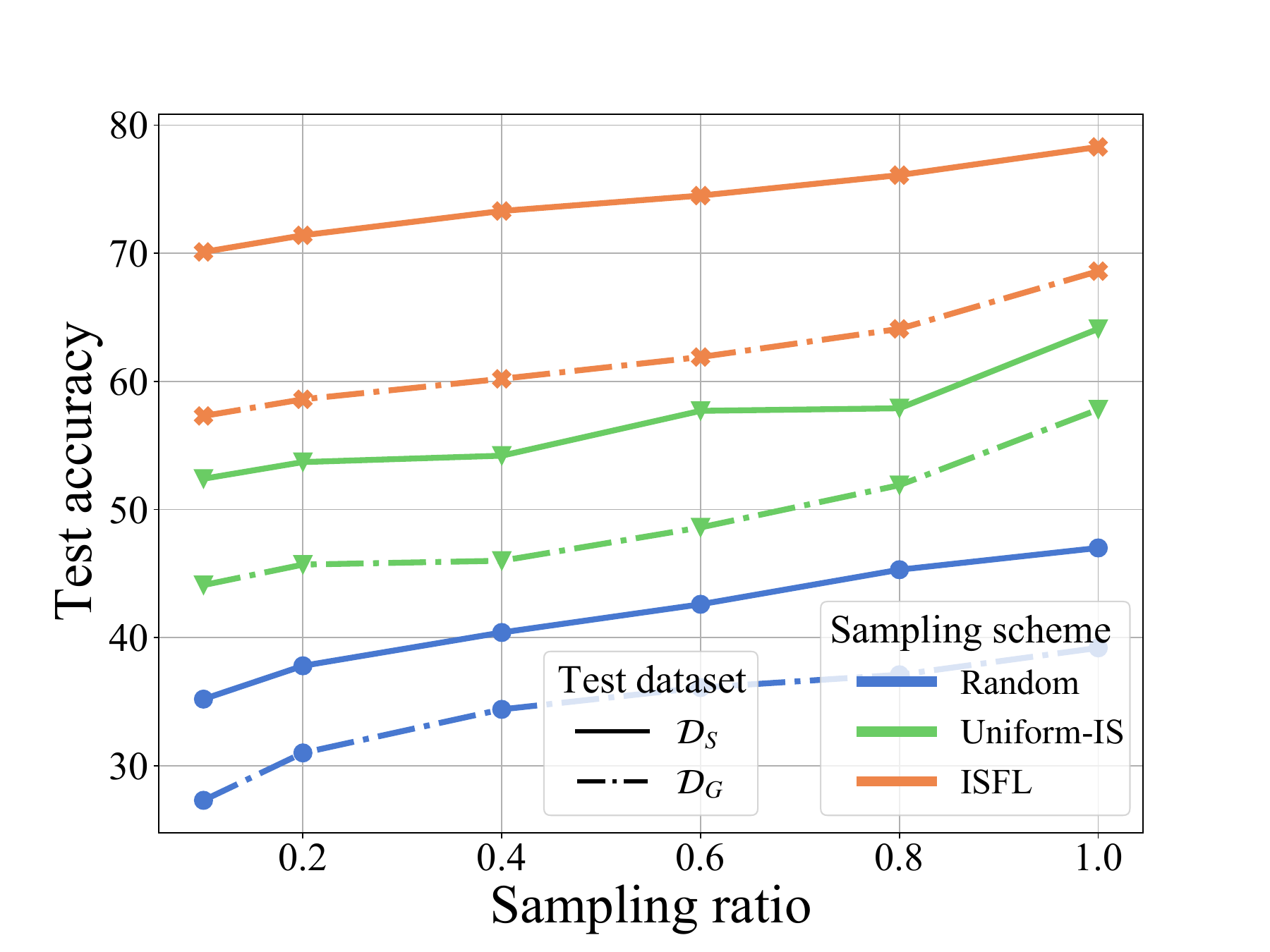}
        \end{minipage}%
        \caption{The sampling efficiency evaluation in different schemes. ($N\!R$=0.98)}
        \label{fig-sr}
\end{figure}

Since ISFL is a sampling-based approach, we also evaluate its data efficiency. Different from training the full data in every local epoch, we consider the partial sampling with a sampling ratio $S\!R$, i.e., each client selects only $\lfloor S\!R\cdot|\mathcal{D}_k|\rfloor$ samples for training at every local epoch. The comparison results with RW-IS and random sampling are shown in Fig. \ref{fig-sr}. ISFL not only gets a higher accuracy, but also its performance is stable and robust. The accuracy decay is little with the $S\!R$ decreases. Particularly, even though each client only samples 0.1 ratio training data in local epochs, ISFL still preserves 58\%, 71\% accuracy on $\mathcal{D}_S$ and $\mathcal{D}_G$, respectively. Thus, ISFL exhibits a better sampling efficiency and data robustness, which is important for the scenarios where the local computation resources/timeliness are restricted.

Overall, the evaluations demonstrate that ISFL improves the original FL on model accuracy, stability, convergence, sampling efficiency, as well as the compatibility with NN models. The experimental results also meet the theoretical deductions based on which the optimal IS weights are derived. The core idea of ISFL strategies can be interpreted as: 1) to up-sample the categories with larger global proportion or 2) to down-sample the categories with larger empirical gradient Lipschitz. For above reasons, we consider ISFL as an effective FL scheme with theoretical guarantees for non-i.i.d. data.
\section{Conclusion}
\label{section conclusion}
In this paper, we investigated the label-skewed non-i.i.d. data quagmire in federated learning and proposed a local importance sampling based framework, ISFL. We derived the theoretical guarantees of ISFL, as well as the solutions of the optimal IS weights using a water-filling approach. In particular, we improved the theoretical compatibility with the neural network models by relaxing the convexity assumptions. The ISFL algorithm was also developed and evaluated on the CIFAR-10 \& CIFAR-100 datasets. The experimental results demonstrated the improvement of ISFL for the non-i.i.d. data from several aspects. ISFL not only promotes the accuracy of basic FL, but also shows its potentials on convergence, explainability, sampling efficiency and data robustness. This is significant for FL, especially for the mobile light-weight scenarios.

For future work, more tests on the theorems and the performance of ISFL can be investigated. For instance, the finer-grained analysis of the parameter deviation and the gradient Lipschitz approximation can be studied. The other types of non-i.i.d. settings such as feature based non-i.i.d. cases shall also be considered. The potential strengths for local personalization can be further developed by setting the $\{p_i\}$ different for each client.
In addition, since ISFL is a local sampling federated scheme, it can be easily migrated into other emerging FL frameworks as a key operation block.
% \input{6appendix.tex}
% \nocite{*}
\bibliographystyle{IEEEtran}
\bibliography{IEEEabrv, refs}
\appendices

\section{Proof of The Convergence Theorem}
\label{appendix A}
To derive the convergence result in Theorem \ref{thm1-bound}, we firstly show some preliminaries and then complete the proofs.
\subsection{Proof Preparations}
Ahead of the proofs, we define some extra shorthand notations for the simplicity of writing as follows.

$t_c$: The latest aggregation epoch, i.e., $t_c=\lfloor\frac{t}{E_l}\rfloor\cdot E_l$.

$\bar{\boldsymbol{\theta}}_{t}=\sum\limits_{k=1}^K\pi_k\boldsymbol{\theta}^k_{t}$: The aggregated model at $t$-epoch. Note that the global model is equal to $\bar{\boldsymbol{\theta}}_{t}$ only at each global rounds, i.e., when $t=t_c$.

$\ell\big(\bar{\boldsymbol{\theta}}_{t}\big)=\sum\limits_{i=1}^Cp_i\ell\big(\xi_i;\bar{\boldsymbol{\theta}}_{t}\big)$: The average loss of the aggregated model $\bar{\boldsymbol{\theta}}_{t}$ on the full data.

$\tilde{\boldsymbol{g}}^{(k,i)}_t$: The stochastic gradient of client $k$ on category $i$ at $t$-th epoch. Let $\boldsymbol{g}^{(k,i)}_t=\mathbb{E}\ \tilde{\boldsymbol{g}}^{(k,i)}_t$ denote its expectation.

$\tilde{\boldsymbol{g}}^{k}_t=\sum\limits_{i=1}^C q^k_i\tilde{\boldsymbol{g}}^{(k,i)}_t$: The average stochastic gradient of client $k$ over all categories at $t$-th epoch. Let $\boldsymbol{g}^{k}_t=\mathbb{E}\ \tilde{\boldsymbol{g}}^{k}_t$ be its expectation.

$\tilde{\boldsymbol{g}}_t=\sum\limits_{k=1}^K \pi_k\tilde{\boldsymbol{g}}^{k}_t$: The aggregated stochastic gradient of all client $k$ at $t$-th epoch. Let $\boldsymbol{g}_t=\mathbb{E}\ \tilde{\boldsymbol{g}}_t$ be its expectation.

Then, we introduce the following necessary lemma supporting the proof.

\begin{lemma}
    \label{lemma0}
    The expectation of the model deviation between $\boldsymbol{\theta}^k_{t}$ and $\bar{\boldsymbol{\theta}}_{t}$ can be upper-bounded by:
    \begin{equation}
        \label{lemma0eq}
        \mathbb{E}\big\|\boldsymbol{\theta}^k_{t}-\bar{\boldsymbol{\theta}}_{t}\big\|^2\leq2\eta^2E_l\phi_k(t),
    \end{equation}
    where $\phi_k(t)=\sum\limits_{\tau=t_c}^{t-1}\Big[2G^2(\tau)+\sigma_k^2(\tau) +\sum\limits_{l=1}^K\pi_l\sigma_l^2(\tau)\Big]$.
    % \begin{equation}
    %     \label{lemma0eq1}
    %     \phi_k(t)=\sum\limits_{\tau=t_c}^{t-1}\Big(G_k^2(\tau)+\sigma_k^2(\tau) +\sum\limits_{l=1}^K\pi_l\big(G_l^2(\tau)+\sigma_l^2(\tau)\big)\Big).
    % \end{equation}
\end{lemma}
\begin{proof}
    \label{lemma0prf}
    According to the FL updating rules, we have:
    \begin{equation}
        \label{lemma0prf-eq1}
        \boldsymbol{\theta}^k_{t}=\bar{\boldsymbol{\theta}}_{t_c}-\sum_{\tau=t_c}^{t-1}\eta\tilde{\boldsymbol{g}}^k_\tau.
    \end{equation}
    Then, from \eqref{eq-fl}, the model deviation can be represented by:
    \begin{align}
        \mathbb{E}\big\|\boldsymbol{\theta}^k_{t}-\bar{\boldsymbol{\theta}}_{t}\big\|^2 & =\mathbb{E}\Big\|\sum_{\tau=t_c}^{t-1}\eta\tilde{\boldsymbol{g}}^k_\tau-\sum_l\pi_l\sum_{\tau=t_c}^{t-1}\eta\tilde{\boldsymbol{g}}^l_\tau\Big\|^2\notag                                                                    \\
                                                                                        & \overset{\text{\ding{192}}}{\leq}2\Big(\mathbb{E}\Big\|\sum_{\tau=t_c}^{t-1}\eta\tilde{\boldsymbol{g}}^k_\tau\Big\|^2+\mathbb{E}\Big\|\sum_l\pi_l\sum_{\tau=t_c}^{t-1}\eta\tilde{\boldsymbol{g}}^l_\tau\Big\|^2\Big)\notag \\
                                                                                        & \overset{\text{\ding{193}}}{\leq}2\eta^2\Big(\mathbb{E}\Big\|\sum_{\tau=t_c}^{t-1}\tilde{\boldsymbol{g}}^k_\tau\Big\|^2+\sum_l\pi_l\mathbb{E}\Big\|\sum_{\tau=t_c}^{t-1}\tilde{\boldsymbol{g}}^l_\tau\Big\|^2\Big)\notag   \\
                                                                                        & \overset{\text{\ding{194}}}{\leq}2\eta^2E_l\sum_{\tau=t_c}^{t-1}\Big(\mathbb{E}\Big\|\tilde{\boldsymbol{g}}^k_\tau\Big\|^2+\sum_l\pi_l\mathbb{E}\Big\|\tilde{\boldsymbol{g}}^l_\tau\Big\|^2\Big)\notag                     \\
                                                                                        & \overset{\text{\ding{195}}}{\leq}2\eta^2E_l\phi_k(t).
    \end{align}
    Therein, \ding{192} and \ding{194} adopt the inequality of means, \ding{193} is the result of Jensen's inequality, \ding{195} can be obtained by Assumption \ref{ass2} and the fact $0< t-t_c\leq E_l$.
\end{proof}

\subsection{Proof of Theorem \ref{thm1-bound}}
With the help of above preliminaries, we now complete the proof of Theorem \ref{thm1-bound}.
\begin{proof}[Proof of Theorem \ref{thm1-bound}]
    \label{thm1-prf}
    Let $\mathcal{T}$ denote a $T$-step time period from $T_0$ to $T_1$, the gradient Lipschitz of category $i$ on client $k$'s local model goes to $L_{k,i}(t)=L_{k,i}(\mathcal{T})$ for $T_0\leq t\leq T_1$. According to Assumption \ref{ass1} and the corollary of L-smoothness ($f(y)\leq f(x)+\langle y-x,\nabla f(x)\rangle+\frac{L}{2}\|y-x\|^2$), we have
    \begin{align}
        \ell(\bar{\boldsymbol{\theta}}_{t+1}) & -\ell(\bar{\boldsymbol{\theta}}_{t}) =\sum\limits_{i=1}^C p_i\left[\ell\big(\xi_i;\bar{\boldsymbol{\theta}}_{t+1}\big)-\ell\big(\xi_i;\bar{\boldsymbol{\theta}}_{t}\big)\right]\notag                                                               \\
                                              & \qquad \leq \left\langle\bar{\boldsymbol{\theta}}_{t+1}-\bar{\boldsymbol{\theta}}_{t},\nabla\ell(\bar{\boldsymbol{\theta}}_{t})\right\rangle+\frac{\bar{L}_\mathcal{T}}{2}\|\bar{\boldsymbol{\theta}}_{t+1}-\bar{\boldsymbol{\theta}}_{t}\|^2\notag \\
                                              & \qquad = -\eta\left\langle\tilde{\boldsymbol{g}}_t,\nabla\ell(\bar{\boldsymbol{\theta}}_{t})\right\rangle+\frac{\eta^2\bar{L}_\mathcal{T}}{2}\|\tilde{\boldsymbol{g}}_t\|^2\label{thm1prf-eq1}
    \end{align}
    where $\bar{L}_\mathcal{T}=\sum_i p_i L_{i}(\mathcal{T})$ is the average gradient Lipschitz of the global model. According to the equality $\mathbb{E}\|X\|^2=\mathbb{E}\|X-\mathbb{E}X\|^2+\|\mathbb{E}X\|^2$, the expectation of $\|\tilde{\boldsymbol{g}}_t\|^2$ in RHS in \eqref{thm1prf-eq1} can be rewritten as:
    \begin{align}
        \mathbb{E}\|\tilde{\boldsymbol{g}}_t\|^2 & =\mathbb{E}\Big\|\sum_{k=1}^K\pi_k\tilde{\boldsymbol{g}}^k_t\Big\|^2\notag                                                                        \\
                                                 & =\mathbb{E}\Big\|\sum_{k=1}^K\pi_k(\tilde{\boldsymbol{g}}^k_t-\boldsymbol{g}^k_t)\Big\|^2+\Big\|\sum_{k=1}^K\pi_k\boldsymbol{g}^k_t\Big\|^2\notag \\
        %  & \overset{\large{\textcircled{\small{1}}}}{\leq} \sum_{k=1}^K\pi_k\mathbb{E}\|\tilde{\boldsymbol{g}}^k_t-\boldsymbol{g}^k_t\|^2+\Big\|\sum_{k=1}^K\pi_k\boldsymbol{g}^k_t\Big\|^2\notag \\
                                                 & \overset{\text{\ding{192}}}{\leq} \sum_{k=1}^K\pi_k\sigma_k^2(t)+\Big\|\sum_{k=1}^K\pi_k\boldsymbol{g}^k_t\Big\|^2,\label{thm1prf-eq2}
    \end{align}
    where \ding{192} adopts Jensen's inequality and Assumption \ref{ass2}. For the first term of RHS in \eqref{thm1prf-eq1}, according to the equality $2\langle a,b\rangle=\|a\|^2+\|b\|^2-\|a-b\|^2$, its expectation can be transformed into:
    \begin{align}
        -\eta\left\langle\tilde{\mathbb{E}\boldsymbol{g}}_t,\nabla\ell(\bar{\boldsymbol{\theta}}_{t})\right\rangle & =-\eta\Big\langle\sum\limits_{k=1}^K\pi_k\boldsymbol{g}^k_t,\nabla\ell(\bar{\boldsymbol{\theta}}_{t}\Big\rangle\notag                          \\
                                                                                                                   & =-\frac{\eta}{2}\big\|\sum\limits_k\pi_k\boldsymbol{g}^k_t\big\|^2-\frac{\eta}{2}\big\|\nabla\ell(\bar{\boldsymbol{\theta}}_{t})\big\|^2\notag \\
                                                                                                                   & \quad+\frac{\eta}{2}\big\|\sum\limits_k\pi_k\boldsymbol{g}^k_t-\nabla\ell(\bar{\boldsymbol{\theta}}_{t})\big\|^2.\label{thm1prf-eq3}
    \end{align}
    Then, considering the last term in \eqref{thm1prf-eq3}, we obtain:
    \begin{align}
         & \big\|\sum\limits_k\pi_k\boldsymbol{g}^k_t-\nabla\ell(\bar{\boldsymbol{\theta}}_{t})\big\|^2=\Big\|\sum_{k}\pi_k\sum_iq^k_i\boldsymbol{g}^{(k,i)}_t\notag                                                       \\
         & \qquad\qquad\qquad\qquad\qquad\qquad-\sum_ip_i\nabla\ell(\xi_i;\bar{\boldsymbol{\theta}}_{t})\Big\|^2\notag                                                                                                     \\
         & \qquad \overset{\text{\ding{192}}}{\leq}\sum_{k}\pi_k\Big\|\sum_iq^k_i\big[\boldsymbol{g}^{(k,i)}_t-\nabla\ell(\xi_i;\bar{\boldsymbol{\theta}}_{t})\big]\notag                                                  \\
         & \qquad\qquad\qquad\quad+\sum_i(p_i-q^k_i)\nabla\ell(\xi_i;\bar{\boldsymbol{\theta}}_{t})\Big\|^2\notag                                                                                                          \\
         & \qquad\overset{\text{\ding{193}}}{\leq}\sum_{k}\pi_k\Big(\Big\|\sum_iq^k_i\big[\boldsymbol{g}^{(k,i)}_t-\nabla\ell(\xi_i;\bar{\boldsymbol{\theta}}_{t})\big]\Big\|\notag                                        \\
         & \qquad\qquad\qquad\quad+\sum_i|p_i-q^k_i|\Big\|\nabla\ell(\xi_i;\bar{\boldsymbol{\theta}}_{t})\Big\|\Big)^2\notag                                                                                               \\
         & \qquad\overset{\text{\ding{194}}}{\leq}\sum_{k}\pi_k\Big(1+\sum_i(q^k_i-p_i)^2\Big)\notag                                                                                                                       \\
         & \qquad\qquad\cdot\Big(\Big\|\sum_iq^k_i\big[\boldsymbol{g}^{(k,i)}_t-\nabla\ell(\xi_i;\bar{\boldsymbol{\theta}}_{t})\big]\Big\|^2\notag                                                                         \\
         & \qquad\qquad\quad+\sum_i\big\|\nabla\ell(\xi_i;\bar{\boldsymbol{\theta}}_{t})\big\|^2\Big)\notag                                                                                                                \\
         & \qquad\overset{\text{\ding{195}}}{\leq}\sum_{k}\pi_k\Big(1+\sum_i(q^k_i-p_i)^2\Big)\notag                                                                                                                       \\
         & \qquad\qquad\cdot\Big(\sum_iq^k_iL_{k,i}^2(t)\big\|\boldsymbol{\theta}^k_{t}-\bar{\boldsymbol{\theta}}_{t}\big\|^2+\sum_i\big\|\nabla\ell(\xi_i;\bar{\boldsymbol{\theta}}_{t})\big\|^2\Big).\label{thm1prf-eq4}
    \end{align}
    Therein \ding{192} is the result of Jensen's inequality, \ding{193} takes the triangular inequality of 2-norm, \ding{194} adopts Cauchy-Schwarz inequality, \ding{195} holds according to Jensen's inequality and Assumption \ref{ass1}. Furthermore, denote $\mathbb{E}\|\nabla\ell(\xi_i;\bar{\boldsymbol{\theta}}_{t})\|\leq G(t)$, and through Lemma \ref{lemma0}, we have:
    \begin{align}
         & \mathbb{E}\big\|\sum\limits_k\pi_k\boldsymbol{g}^k_t-\nabla\ell(\bar{\boldsymbol{\theta}}_{t})\big\|^2\leq 2CG^2(t)\notag  \\
         & \quad+\sum_{k}2\pi_k\eta^2E_l\phi_k(t)\Big[\Big(1+\sum_i(q^k_i-p_i)^2\Big)\sum_iq^k_iL_{k,i}^2(t)\Big].\label{thm1prf-eq5}
    \end{align}
    Then, combining \eqref{thm1prf-eq2}-\eqref{thm1prf-eq5}, \eqref{thm1prf-eq1} goes to:
    \begin{align}
        \mathbb{E} & \big[\ell(\bar{\boldsymbol{\theta}}_{t+1}) -\ell(\bar{\boldsymbol{\theta}}_{t})\big] \leq\frac{\eta}{2}(\eta\bar{L}_\mathcal{T}-1)\Big\|\sum_{k=1}^K\pi_k\boldsymbol{g}^k_t\Big\|^2\notag \\
                   & +\frac{\eta^2\bar{L}_\mathcal{T}}{2}\sum_{k=1}^K\pi_k\sigma_k^2(t) -\frac{\eta}{2}\big\|\nabla\ell(\bar{\boldsymbol{\theta}}_{t})\big\|^2+\eta CG^2(t)\notag                              \\
                   & +\sum_{k=1}^K\pi_k\eta^3E_l\phi_k(t)\Big[\Big(1+\sum_{i=1}^C(q^k_i-p_i)^2\Big)\sum_{i=1}^Cq^k_iL_{k,i}^2(t)\Big].\label{thm1prf-eq6}
    \end{align}
    Since the practical learning rates of neural networks are commonly recommended to be set as a small value such as 1e-3 for better training stability. For proper small learning rate satisfying $\eta\bar{L}_\mathcal{T}\leq 1$, the first term in \eqref{thm1prf-eq6} can be omitted. Consider the time average $\frac{1}{T}\sum_{t=T_0}^{T_1}[\cdot]$ for both sides of \eqref{thm1prf-eq6} from $T_0$ to $T_1$ and denote the optimum of $\mathbb{E}\ell(\bar{\boldsymbol{\theta}}_{T_1})$ as $\ell^*$, we obtain:
    \begin{align}
        \frac{1}{T}\sum_{t=T_0}^{T_1}\mathbb{E} & \big[\ell(\bar{\boldsymbol{\theta}}_{t+1}) -\ell(\bar{\boldsymbol{\theta}}_{t})\big]=\frac{1}{T}\big[\ell^* -\ell(\bar{\boldsymbol{\theta}}_{T_0})\big]\notag \\
                                                & \leq-\frac{\eta}{2T}\sum_{t=T_0}^{T_1}\big\|\nabla\ell(\bar{\boldsymbol{\theta}}_{t})\big\|^2+\frac{\eta}{2}\psi(\mathcal{T})\notag                           \\
                                                & \quad +\frac{1}{T}\sum_{t=T_0}^{T_1}\sum_{k=1}^K\pi_k\eta^3E_l\phi_k(t)\rho_\mathcal{T}(\boldsymbol{q}^k).\label{thm1prf-eq7}
    \end{align}
    Finally, taking the term of gradient norm to the left-hand side, \eqref{thm1prf-eq7} leads to \eqref{thm1-ineq}.
    % \begin{align}
    %     \frac{1}{T}\sum_{t=T_0}^{T_1}\big\|\nabla\ell(\bar{\boldsymbol{\theta}}_{t})\big\|^2\leq & \frac{2\big(\ell(\bar{\boldsymbol{\theta}}_{T_0})-\ell^*\big)}{\eta T}+\psi(\mathcal{T})\notag       \\
    %                                                                                              & +\frac{2\eta^2E_l}{T}\sum_{t=T_0}^{T_1}\sum_{k=1}^K\pi_k\phi_k(t)\rho_\mathcal{T}(\boldsymbol{q}^k).
    % \end{align}
\end{proof}

\section{Proofs of The Optimal IS Strategies}
\label{appendix B}
\subsection{Proof of Theorem \ref{thm2-opt}}
We firstly derive the Theorem \ref{thm2-opt}, which gives the theoretical solution of the optimal IS probabilities.
\label{appendix B-Thm2}
\begin{proof}[Proof of Theorem \ref{thm2-opt}]
    \label{thm2prf}
    The theorem is the solution to the sub-problem $\mathcal P^k_{\mathcal{T}}$. For simplicity, we denote the two terms of $\rho_{\mathcal{T}}(\boldsymbol{q}^k)$ respectively by
    \begin{subequations}
        \begin{align}
            \label{thm2prf-AB}
            A(\boldsymbol{q}^k) & =1+\sum_{i=1}^C(p_i-q^k_i)^2,              \\
            B(\boldsymbol{q}^k) & =\sum_{i=1}^C q^k_iL^2_{k,i}(\mathcal{T}).
        \end{align}
    \end{subequations}
    It is easy to check that the objective function is convex on $\boldsymbol{q}^k$. Therefore, to derive the solution of $\mathcal P^k_{\mathcal{T}}$, we formulate the Lagrangian function:
    \begin{equation}
        \label{thm2prf-lagrangian}
        \mathcal{L}=A(\boldsymbol{q}^k)\cdot B(\boldsymbol{q}^k)-\lambda\Big(\sum\limits_{i=1}^C q_i-1\Big)-\sum\limits_{i=1}^C\mu_i\big(q_i^k-\varpi p_i^k\big),
    \end{equation}
    where the scalar $\lambda$ is the Lagrangian multiplier for the equality constraint and the vector multiplier $\boldsymbol{\mu}=[\mu_1,\cdots,\mu_C]$ is for the inequality constraint. Hence, the Karush-Kuhn-Tucker (KKT) conditions of $\mathcal P^k_{\mathcal{T}}$ can be written as
    \begin{subnumcases}{}
        \frac{\partial \mathcal{L}}{\partial q_j^k}=2(q_j^k-p_j)B+L^2_{k,j}(\mathcal{T})A-\lambda-\mu_j=0 \label{thm2prf-kkt1}\\
        \sum\limits_{i=1}^C q_i^k=1\label{thm2prf-kkt2}                                                                           \\
        \mu_j\geq 0\label{thm2prf-kkt3}                                                                                         \\
        q_j^k\geq \varpi p_j^k\label{thm2prf-kkt4}                                                                                \\
        \mu_j(\varpi p_j^k-q_j^k)=0\label{thm2prf-kkt5}
    \end{subnumcases}
    where $j=1,\cdots,C$. Note that for \eqref{thm2prf-kkt3}-\eqref{thm2prf-kkt5}, $\mu_j=0$ holds if $q_j^k> \varpi p_j^k$. Otherwise, $q_j^k= \varpi p_j^k$ and $\mu_j>0$. Thus, we are able to adopt a water-filling method to solve the KKT conditions with inequality constraints. Specifically, we can firstly assume that $\mu_j=0$ for all $j$ and then adjust each $q^{k*}_j$ to satisfy \eqref{thm2prf-kkt4}. Since $A>0$ always holds, according to \eqref{thm2prf-kkt1}, each $q_j^k$ can be represented as
    \begin{equation}
        \label{thm2prf-qstar}
        q_j^k=p_j+\frac{\lambda-L^2_{k,j}(\mathcal{T})A}{2B}
    \end{equation}
    By summing $j$ up all \eqref{thm2prf-kkt1}, we obtain
    \begin{equation}
        \label{thm2prf-A}
        A=\frac{C\lambda}{\sum\limits_{i=1}^CL^2_{k,i}(\mathcal{T})}
    \end{equation}
    % Then, from the definition of $\Psi$, we can obtain the equation:
    % \begin{equation}
    %     \label{thm2prf-solvephi}
    %     \Psi=1+\frac{\sum\limits_{i=1}^C \lambda^2\Bigg(1-\frac{CL_j^2}{\sum\limits_{i=1}^CL_i^2}\Bigg)^2}{4\eta^2\Phi^2}
    % \end{equation}
    Then, from the definition of $A$, $B$ can be solved as
    \begin{equation}
        \label{thm2prf-B}
        B=\frac{\lambda}{2}\sqrt{\frac{\sum\limits_{m=1}^C \bigg(1-\frac{CL^2_{k,m}(\mathcal{T})}{\sum_iL^2_{k,i}(\mathcal{T})}\bigg)^2}{A-1}}
    \end{equation}
    By substituting \eqref{thm2prf-A} and \eqref{thm2prf-B} into \eqref{thm2prf-qstar}, we obtain
    \begin{align}
        q_j^{k*} & =p_j+\frac{1-\frac{CL^2_{k,j}(\mathcal{T})}{\sum_iL^2_{k,i}(\mathcal{T})}}{\sqrt{\sum\limits_{m=1}^C \bigg(1-\frac{CL^2_{k,m}(\mathcal{T})}{\sum_iL^2_{k,i}(\mathcal{T})}\bigg)^2}}\cdot\sqrt{A-1}\notag \\
                 & =p_j+\alpha^k_{j}(\mathcal{T})\cdot\Gamma_k(\lambda).\label{thm2prf-qstar-final}
    \end{align}
    Here, we obtain the solutions to $\mathcal P^k_{\mathcal{T}}$ in Theorem \ref{thm2-opt}.
\end{proof}

\subsection{Proof of Theorem \ref{thm3-opt-Gamma}}
Theorem \ref{thm3-opt-Gamma} directs the calculation of the coefficient $\Gamma$ to obtain the optimal IS weights in practical algorithms. We sketch the proof as follows.
\label{appendix B-Thm3}
\begin{proof}[Proof of Theorem \ref{thm3-opt-Gamma}]
    \label{thm3prf}
    Practically, $\varpi$ is set as a small value to ensure that all categories contribute to the local training. Besides, as the category distribution of all data, each $p_j$ shall not be too small. Thus, it is natural to assume that $p_j\geq\varpi p_j^k$ holds in most cases. Moreover, note that the factors $\{\alpha^k_j(\mathcal{T})\}$ satisfy the zero-sum property. The water-filling approach will then lead to $p_j+\alpha^k_j(\mathcal{T})\Gamma\geq\varpi p_j^k$ for each $j$. That is,

    \begin{subnumcases}{}
        \Gamma  \geq \frac{p_j-\varpi p_j^k}{-\alpha^k_j(\mathcal{T})}\ \ \ \ if\ \ \ \alpha^k_j(\mathcal{T})>0,\label{thm3prf-Gamma-a}\\
        \Gamma  \leq \frac{p_j-\varpi p_j^k}{-\alpha^k_j(\mathcal{T})}\ \ \ \ if\ \ \ \alpha^k_j(\mathcal{T})<0.\label{thm3prf-Gamma-b}
    \end{subnumcases}
    \eqref{thm3prf-Gamma-a} always holds because the right-hand side is smaller than $0$ while $\Gamma\geq 0$. Therefore, the optimal $\Gamma^*$ shall satisfy \eqref{thm3prf-Gamma-b}. By setting $\Gamma$ as the smallest non-negative $\frac{p_j-\varpi p_j^k}{-\alpha^k_j(\mathcal{T})}$, the numerical solutions to $\mathcal P^k_{\mathcal{T}}$ can be obtained.
\end{proof}
\end{document}